\documentclass[lettersize,journal]{IEEEtran}

\usepackage{amsmath,amsfonts,bm}

\def\eqref#1{equation~\ref{#1}}

\def\1{\bm{1}}

\def\vc{{\bm{c}}}

\def\vt{{\bm{t}}}

\def\vv{{\bm{v}}}

\def\vx{{\bm{x}}}

\def\vz{{\bm{z}}}

\def\mA{{\bm{A}}}
\def\mB{{\bm{B}}}
\def\mC{{\bm{C}}}

\def\mG{{\bm{G}}}

\def\mI{{\bm{I}}}

\def\mM{{\bm{M}}}

\def\mZ{{\bm{Z}}}

\DeclareMathAlphabet{\mathsfit}{\encodingdefault}{\sfdefault}{m}{sl}
\SetMathAlphabet{\mathsfit}{bold}{\encodingdefault}{\sfdefault}{bx}{n}

\DeclareMathOperator{\Tr}{Tr}

\usepackage{amsmath,amsfonts}
\usepackage[caption=false,font=normalsize,labelfont=sf,textfont=sf]{subfig}

\usepackage{mathtools}
\usepackage{amsthm}
\usepackage{graphicx}
\usepackage{booktabs}
\usepackage{tabularx}
\usepackage{multirow} 
\usepackage{cite}

\usepackage{cases}
\usepackage{bbm}

\usepackage{enumitem,kantlipsum}
\usepackage{pifont}
\usepackage{makecell}

\usepackage{algorithm}
\usepackage{algpseudocode}
\usepackage[breakable]{tcolorbox}

\newcommand{\myparagraph}[1]{\noindent\textbf{#1.}}
\newcolumntype{P}[1]{>{\centering\arraybackslash}p{#1}}

\def\ie{i.e.}
\def\eg{e.g.}
\DeclareMathOperator{\Diag}{Diag}

\theoremstyle{plain}

\newtheorem{proposition}{Proposition}

\theoremstyle{definition}

\theoremstyle{remark}

\def\f{\boldsymbol{f}}
\def\g{\boldsymbol{g}}

\def\Z{\boldsymbol{Z}}
\def\vdelta{\boldsymbol{\delta}}

\def\0{\boldsymbol{0}}

\def\mck{\color{black}}

\usepackage{colortbl}
\usepackage[dvipsnames]{xcolor}
\definecolor{myGray}{gray}{0.9}

\usepackage[percent]{overpic}

\usepackage{hyperref}
\usepackage{url}

\usepackage{array}
\usepackage{textcomp}
\usepackage{stfloats}
\usepackage{verbatim}
\def\BibTeX{{\rm B\kern-.05em{\sc i\kern-.025em b}\kern-.08em
    T\kern-.1667em\lower.7ex\hbox{E}\kern-.125emX}}
\usepackage{balance}

\begin{document}
\title{Learning Deep Modality-Shared Self-Expressiveness for Image Clustering with Textual Information}
\author{Xianghan Meng\thanks{Xianghan Meng, Wei He, Zhiyuan Huang, and Chun-Guang Li are with the School of Artificial Intelligence, Beijing University of Posts and Telecommunications, Beijing 100876, P.R. China. (e-mail: \{mengxianghan; wei.he; huangzhiyuan; lichunguang\}@bupt.edu.cn).},~Wei He,~Zhiyuan Huang,~and~Chun-Guang Li$^\ast$\thanks{$^\ast$Corresponding author.}
}

\markboth{Journal of \LaTeX\ Class Files,~Vol.~18, No.~9, September~2026}%
{Meng \MakeLowercase{\textit{et al.}}: Learning Deep Modality-Shared Self-Expressiveness for Image Clustering with Textual Information}

\maketitle

\begin{abstract}
Leveraging textual information for image clustering has emerged as a promising direction, largely owing to the powerful representations learned by Vision-Language Models (VLMs).
Existing approaches typically retrieve a textual counterpart for each image and then refine multimodal representations by directly enforcing cross-modal agreement, \eg, maximizing image-text similarity inherited from pretrained VLMs.
However, such a strategy %
aligns heterogeneous representations across modalities without explicitly modeling the intrinsic structure within each modality and thus might yield unreliable alignment %
or %
distort modality-specific structures that are crucial for clustering.
In this paper, we propose a simple but principled approach, termed deep \underline{mo}dality-sha\underline{r}ed \underline{s}elf-\underline{e}xpressive model (DeepMORSE), which discovers cross-modal structures %
via a modality-shared self-expressive model and simultaneously learns structured representations that conform to a union of modality-specific subspaces. 
Moreover, we theoretically justify that the modality-shared self-expressive coefficients suppress inter-class noise towards a subspace-preserving solution, and %
show that mini-batch optimization procedure introduces an implicit regularization onto the self-expressive model. %
We evaluate our DeepMORSE on six widely used image clustering benchmarks and observe performance improvements exceeding 3\% on the UCF-101, DTD-47, and ImageNet-Dogs datasets.
In addition, we demonstrate the strong transferability of the learned representations by achieving state-of-the-art performance on downstream tasks such as image retrieval and zero-shot classification--—without requiring any task-specific losses or post-processing.
\texttt{The code is available at: \url{https://github.com/mengxianghan123/DeepMORSE}.}
\end{abstract}


\section{Introduction}
\label{sec:introduction}

\IEEEPARstart{N}{owadays} vast quantities of unlabeled data are generated at an unprecedented scale. The ability to discover %
meaningful latent structures from massive high-dimensional data without human annotation has become a central challenge in various fields of data science. %
A fundamental approach to tackling this challenge is clustering, which aims to partition data according to their latent patterns or geometric structures without relying on ground-truth labels~\cite{Forgy:Biometrics1965,MacQueen-1967}.
The geometric structures, referring to the topology, distribution, and relative positional relationships of data in the ambient space, are often modeled as, \eg, centroids~\cite{Forgy:Biometrics1965,MacQueen-1967,Arthur_Vassilvitskii_2007}, linear (or affine) subspaces~\cite{Vidal:SPM11-SC,Vidal:Springer16,Li:TIP17}, or manifolds~\cite{Patel:ICIP14,Elhamifar:NIPS11,Li:Arxiv22,Ding:ICCV23}.
Although these approaches perform well on simple datasets (\eg, MNIST~\cite{Lecun:MNIST}), they often struggle with high-dimensional real-world images that exhibit complex structures~\cite{Krizhevsky:2009-CIFAR,Deng:CVPR09}.

\begin{figure}[t]
    \centering
    \includegraphics[trim=220pt 120pt 270pt 100pt, clip,width=\linewidth]{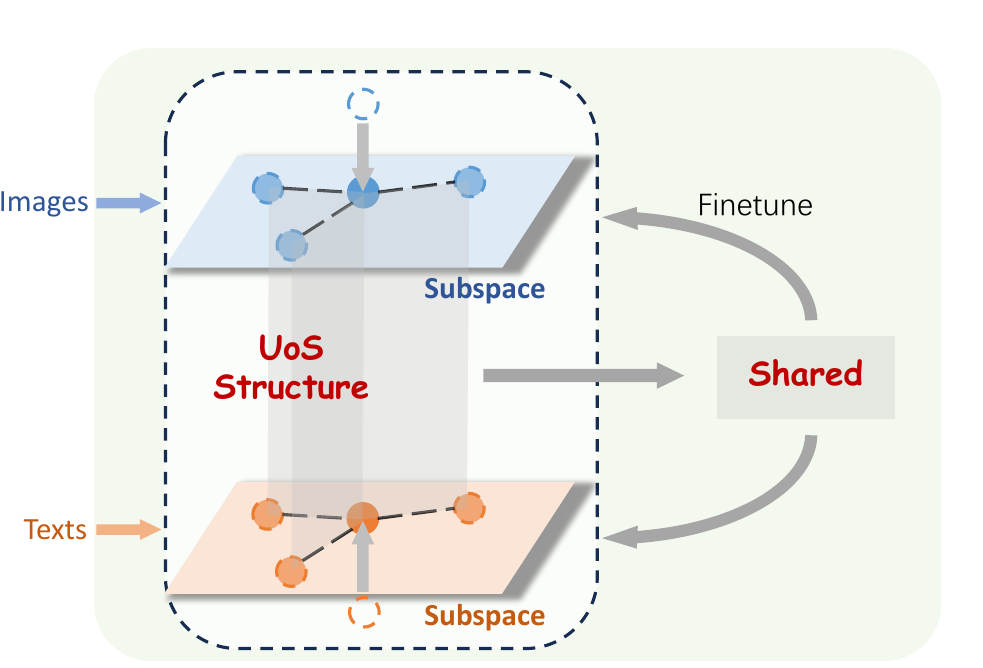}
    \caption{\textbf{Illustration of the basic idea of the paper.} We explicitly learn structured representations that conform to a Union of Subspaces (UoS) and capture %
    modality-shared structure for clustering.}
    \label{fig:fig1}
\end{figure}

Over the past decade, deep clustering has emerged as a promising line of research which leverages the representation learning ability of deep neural networks~\cite{Lu:Vicinagearth24,Ren:TNNLS25} and has quickly progressed following the rapid advances in representation learning---from early autoencoder-based approaches~\cite{Peng:IJCAI16, Xie:ICML16-DEC, Guo:IJCAI17-IDEC, Yang:ICML17-DCN} to contrastive learning-based approaches~\cite{Van:ECCV2020-SCAN,Li:AAAI21-CC,Zhong:ICCV21-GCC,Dang:CVPR21-NNM,Niu:TIP22-spice}, and more recently, to pretrained Vision-Language Models (VLMs)-based multimodal approaches~\cite{Adaloglou:BMVC23,Cai:AAAI23-SIC,Chu:ICLR24,Li:ICML24-TAC}.
Despite the rapid advances in clustering performance, fundamental questions naturally arise in the multimodal approaches: \emph{a) what latent patterns in data representation would be shared across different modalities? b) how to properly learn the latent patterns across multimodal representations?}

Existing multimodal clustering approaches typically address these questions by directly enforcing cross-modal alignment, %
\eg, by maximizing the similarity of pseudo-labels across modalities~\cite{Cai:AAAI23-SIC,Li:ICML24-TAC}.
However, without explicitly modeling the intrinsic structure of each modality, the alignment of heterogeneous distributions between different modalities may distort modality-specific structures, %
yielding %
representations with unclear structure.
This issue is particularly demonstrated by~\cite{Yi:CVPR24, Mistretta:ICLR25}, which show that the off-the-shelf VLMs provide unreliable intra-modal geometry, leading to suboptimal performance for intra-modal tasks (\eg, clustering and retrieval). %

Motivated by the intrinsic heterogeneity~\cite{Liang:NIPS23,Liang:ACM24} and the modality gap~\cite{Liang:NIPS22,Schrodi:ICLR25} between multimodal representations, it is natural 
to %
permit each modality to have %
its own modality-specific structure.
On the other hand, recent empirical evidence suggests that the representations of different modalities often exhibit a shared geometric structure as model scale increases~\cite{Huh:ICML24,Megan:ICML25,Groger:arxiv26}.
Putting %
two aspects together, it is appealing to build a learning %
model that enables: %
a) modality-specific structure %
flexible to account for the heterogeneity and %
b) consistent partition patterns %
shared across different modalities.

In this paper, we introduce a union of modality-specific low-dimensional subspaces to accommodate %
the heterogeneity across different modalities\footnote{The union-of-subspaces assumption of deep representations is well-grounded in recent findings \cite{Mamou:ICML20,Liu:ML22,Bradley:ICLR23,Trager:ICCV23,Pelleriti:ICML25}.}, while assuming that %
consistent partition patterns are shared among them, as illustrated in Figure \ref{fig:fig1}.
Specifically, we propose a simple but principled approach, termed \textbf{deep} \underline{\textbf{mo}}dality-sha\underline{\textbf{r}}ed \underline{\textbf{s}}elf-\underline{\textbf{e}}xpressive model (DeepMORSE) to %
learn structured multimodal representations that %
align with a union of modality-specific subspaces and simultaneously capture the %
partition patterns shared across different modalities for clustering.
More specifically, our DeepMORSE 
integrates a shared deep self-expressive model %
to discover modality-invariant self-expressive coefficients and learn modality-specific structured representations that %
tend to form the union of modality-specific subspaces. 
Moreover, we provide theoretical results to justify the merit of learning self-expressive coefficients across different modalities, the guarantee to yield non-collapsed modality-specific representations, and the induced implicit regularization in implementation. 
To evaluate the performance of our DeepMORSE, we conduct extensive experiments on six widely used image clustering benchmarks to 
demonstrate significant performance improvements, validate the theoretical results, and provide insightful and comprehensive evaluations. %
In addition, we also show the strong transferability of the learned representations on downstream tasks such as image retrieval and zero-shot classification without requiring any task-specific losses or post-processing. %

The main contributions of the paper are highlighted below. %
\begin{enumerate}[leftmargin=*]
    \item %
    We propose DeepMORSE, which jointly learns structured representations %
    and %
    discovers partition patterns %
    that are shared across different modalities for clustering.

    \item We provide theoretical analysis to justify the underlying mechanism of our DeepMORSE---suppressing inter-class noise and exhibiting an implicit regularization. 
    
    \item We demonstrate the effectiveness of our DeepMORSE with extensive experiments on %
    six image clustering benchmarks.  %
    \item We 
    validate the generalization ability %
    of the learned structured representations on downstream tasks including image retrieval and zero-shot classification. %
\end{enumerate}

\myparagraph{Paper Outline} The remainder of this %
paper is organized as follows. We review related works in Section \ref{Sec:Related Work}, present the formulation and theoretical justifications %
in Section \ref{Sec:method}, describe the experimental setups and results in Section \ref{Sec:Experiments}, and make concluding remarks %
in Section \ref{Sec:Conclusion}.

\section{Related Work}
\label{Sec:Related Work}
This section reviews relevant %
work on deep clustering, self-expressive models and structured representations learning.

\myparagraph{Deep Clustering}
The progress of deep clustering closely follows the rapid advances in unsupervised representation learning. 
The %
deep clustering approaches in the early stage are mainly based on %
autoencoders~\cite{Guo:IJCAI17-IDEC,Yang:ICML17-DCN}. 
Then, with the flourishing success of %
contrastive learning~\cite{Chen:ICML20-SimCLR,He:CVPR20-moco}, a number of deep clustering approaches have been proposed, which are designed by 
integrating diverse self-supervision signals, \eg, pseudo-labeling~\cite{Van:ECCV2020-SCAN}, cluster and graph-level contrastive learning~\cite{Li:AAAI21-CC,Zhong:ICCV21-GCC,Qian:CVPR22-CoKe,Qian:ICCV23-SeCu}, neighbor matching~\cite{Dang:CVPR21-NNM}, prototype scattering~\cite{Huang:TPAMI23-Propos}, and contextually affinitive neighborhood mining~\cite{Yu:NIPS23-CONNR}.
Among these approaches, clusters are merely %
modeled as a set of centroids~\cite{Qian:CVPR22-CoKe,Qian:ICCV23-SeCu}---which can be viewed as a zero-dimensional subspace. %

Recently, driven by contrastive learning, miscellaneous deep multi-view clustering approaches \cite{Tian:ECCV20,Lin:CVPR21,Pan:NIPS21,Trosten:CVPR21} have been presented, which emphasize on maximizing the alignment between different views. 
More recently, %
there are also some works to address the noisy correspondence issue in contrastive multi-view clustering, by \eg, a noise-robust contrastive loss~\cite{Yang:TPAMI22}, cross-sample and cross-view feature aggregation~\cite{Yan:CVPR23}, high-order random walks for global data pair identification~\cite{Lu:AAAI24}, and a spectral %
correspondence refinery framework~\cite{Guo:NeurIPS24-CANDY}. %
Despite these advances, existing deep multi-view clustering approaches primarily lack explicit modeling of the underlying geometric structure of data within each view. %

Since that multimodal foundation models are used to provide powerful representations for various downstream tasks,  
multimodal clustering methods have received increasing attention~\cite{Cai:AAAI23-SIC,Li:ICML24-TAC}. 
These methods typically train a multimodal classifier and align the pseudo-labels of each sample---as well as its neighbors---across different modalities.
To %
mitigate %
the mismatch between different modalities, multi-level cross-modal alignments are introduced with theoretical guarantees on clustering risk~\cite{Qiu:AAAI24}.
Meanwhile, an orthogonal line of work explores training-free language-assisted image clustering by introducing hierarchical caption-noun semantic alignment via optimal transport~\cite{Zhu:AAAI26} and gradient-guided positive noun filtering with provable error bounds~\cite{Peng:ICCV25}.
Nevertheless, without explicitly modeling the intrinsic structure of each modality, it is unclear whether the cross-modal alignment would %
distort the intra-modal structure.

\myparagraph{Self-Expressive Models}
The self-expressive models~\cite{Elhamifar:CVPR09} are designed to recover %
the underlying subspace structure by expressing each data point as a linear combination of others.
Under certain regularizations, \eg, sparsity~\cite{Elhamifar:CVPR09}, %
low-rankness~\cite{Liu:ICML10}, the non-zero expressive coefficients are guaranteed to indicate which data points %
belong to the same subspace~\cite{Elhamifar:TPAMI13, Soltanolkotabi:AS12, Li:JSTSP18}.
After solving the self-expressive coefficients for the given data, spectral clustering~\cite{Shi:TPAMI00} is applied to obtain the subspace clustering result. %
Recent progress in the self-expressive model focuses on improving the scalability~\cite{You:CVPR16-EnSC,Chen:CVPR20,Zhang:CVPR21-SENet}, incorporating powerful representation learning~\cite{Cai:TGRS20,Ma:MM20,Wei:CVPR23, Meng:ICLR25}, and designing white-box deep architectures~\cite{Zhao:CPAL24}. %
It is noteworthy that the %
connection between self-expressive models and the self-attention mechanism has been discussed in \cite{Zhang:CVPR21-SENet, Vidal:CAMSAP25}. 
Nevertheless, these works mentioned above still focus on data with a single modality.

\myparagraph{Learning Structured Representations}
In supervised learning, structured representations that conform to a union of orthogonal subspaces are pioneered in Orthogonal Low-rank Embedding (OL\'{E})~\cite{Lezama:CVPR18} and Maximal Coding Rate Reduction ($\text{MCR}^2$)~\cite{Yu:NIPS20}. 
For both approaches, the learned representations in each class are guaranteed to reside in class-specific subspaces %
and these class-specific subspaces are arranged to be mutually orthogonal \cite{Wang:ICML24}.
In unsupervised learning, an approach called Manifold Linearization and Clustering (MLC) attempts to learn structured representations with a modified $\text{MCR}^2$ principle, where the membership information is replaced by the learned affinity~\cite{Ding:ICCV23}. 
More recently, a principled approach for deep subspace clustering (PRO-DSC)~\cite{Meng:ICLR25} is proposed, which learns structured representations by the self-expressive model with a maximal total coding rate regularization and provides theoretical justifications showing that the learned representations are guaranteed to avoid the catastrophic feature collapse~\cite{Haeffele:ICLR21} and %
tend to form desired structure of a union of subspaces. 
Unfortunately, these works are not designed for dealing with multimodal data. %
Our DeepMORSE can be viewed as a multimodal extension of~\cite{Meng:ICLR25}. %

\section{Our Proposed Approach: DeepMORSE}
\label{Sec:method}

This section presents a deep modality-shared self-expressive model for learning modality-invariant structure and structured representations, then provides theoretical justifications, and describes the textual counterpart generation strategy, scalable implementation with reparameterization.

\begin{figure*}[tb]
    \centering
   \includegraphics[trim=150pt 0pt 100pt 30pt, clip,width=0.9\textwidth]{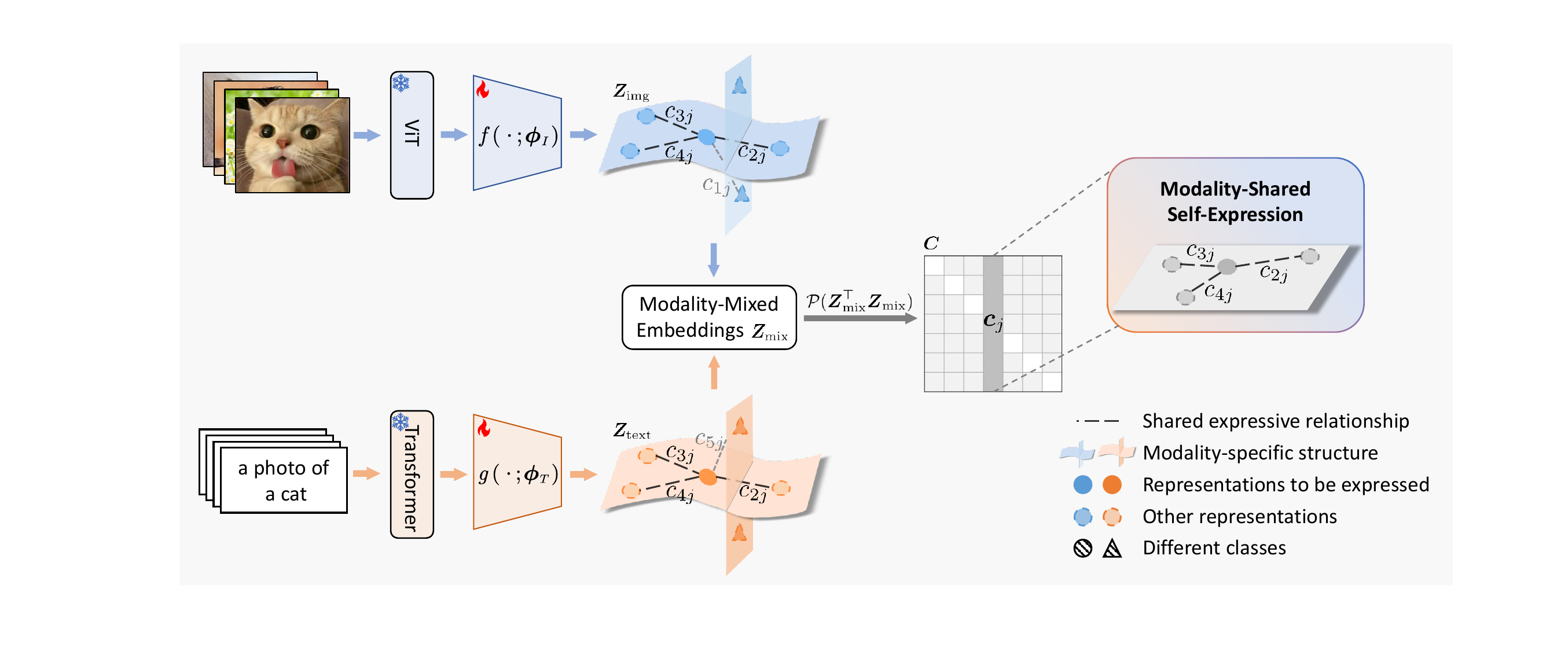}
   \vskip -0.2in
\caption{\mck \textbf{The overview of DeepMORSE.} DeepMORSE learns modality-invariant structure through modality-shared self-expression, and jointly fine-tunes the representations to conform to the union of modality-specific subspace structures.}  
\label{fig:fig2}
\end{figure*}

\subsection{Modality-Shared Deep Self-Expressive Model}

Existing approaches %
to model data relations for %
clustering often rely on neighborhood information, either in a single-modality setting~\cite{Van:ECCV2020-SCAN,Dang:CVPR21-NNM,Niu:TIP22-spice,Huang:TPAMI23-Propos} or in a multimodal setting~\cite{Cai:AAAI23-SIC,Li:ICML24-TAC}.
Although neighborhoods can %
reflect the local distribution of data, they suffer from two fundamental limitations: a) the distance metric---on which neighborhood %
is based---becomes less reliable in high-dimensional spaces~\cite{Beyer:ICDT99,Aggarwal:ICDT01}; and b) %
it is insufficient to capture the global structure unless the data are densely sampled~\cite{Roweis:JMLR03}. 
As an alternative, self-expressive model \cite{Elhamifar:CVPR09} uncovers the global Union-of-Subspaces (UoS) structure of the data by seeking linear combination relationships among data, %
which is suitable for %
high-dimensional and sparsely sampled data~\cite{Vidal:Springer16,Wright-Ma-HD2A}.

\myparagraph{Learning UoS Structures by Self-Expressive Model}
Under the assumption that data points lie on (or nearby) a union of low-dimensional subspaces, self-expressive model \cite{Elhamifar:CVPR09} discovers the UoS structure of the data by solving the %
problem: 
\begin{align}
\label{eq:se_model}
\begin{aligned}
   \mathop{\min}_{\vc_j}~\|\vx_j-\sum_{i\neq j}c_{ij}\vx_i\|_2^2+r(\vc_j),
\end{aligned}
\end{align}
for all $j\in\{1,\cdots,N\}$, where %
$r(\cdot):\mathbb{R}^N\mapsto\mathbb{R}_{+}$ is a regularization on the expressive coefficient $\vc_j=[c_{1j},\cdots,c_{Nj}]^\top$.
The %
widely used regularizer, \eg, $\ell_1$-norm, can guarantee the optimal solution of (\ref{eq:se_model}) under mild condition to enjoy a so-called \emph{subspace-preserving property}, \ie, the non-zero coefficients indicate the data points that belong to the same subspace under certain conditions~\cite{Elhamifar:TPAMI13, Soltanolkotabi:AS12, Li:JSTSP18}.

\myparagraph{Discovering Shared UoS Structures %
across Modalities}
Recent empirical evidence suggests that the learned representations from different modalities tend to exhibit a shared %
geometric structure~\cite{Huh:ICML24,Megan:ICML25,Groger:arxiv26}. 
This motivates us to leverage the information from multiple modalities %
to jointly discover the shared UoS structures. %
To this end, it is natural to introduce a self-expressive model with \emph{modality-invariant coefficients}, 
\ie,
\begin{align}
\label{eq:se_model_share_c}
\begin{aligned}
    \mathop{\min}_{\vc_j}~\|\vx_j-\sum_{i\neq j}c_{ij}\vx_i\|_2^2+\|\vt_j-\sum_{i\neq j}c_{ij}\vt_i\|_2^2+r(\vc_j),
\end{aligned}
\end{align}
for all $j\in\{1,\cdots,N\}$, where $\{\vt_i\}_{i=1}^N$ denote the associated textual data.
The benefit of this modality-invariant design can be understood through the example illustrated in Figure~\ref{fig:fig2}. For the $j$-th data point, the self-expressive coefficients $\{c_{1j},c_{2j},c_{3j},c_{4j}\}$ capture its relations with other images, while $\{c_{2j},c_{3j},c_{4j},c_{5j}\}$ capture its relations among textual counterparts, in which only $\{c_{2j},c_{3j},c_{4j}\}$ reflect correct subspace structure that is shared across two modalities, whereas $c_{1j}$ and $c_{5j}$ arise from modality-specific noise. By enforcing self-expressive coefficients shared across modalities, the formulation (\ref{eq:se_model_share_c}) promotes the coefficients that are consistently subspace-preserving across modalities while suppressing modality-specific perturbations.  %

\myparagraph{Learning Modality-Specific Structured Representations}
For the real-world data, however, the distribution of the raw data in each modality might not conform well to a UoS structure. %
To %
handle more complex data and more complicated distribution structures, %
we employ a set of %
transformations $\f,\g:\mathbb{R}^D\mapsto\mathbb{R}^d$ to obtain modality-specific representations $\vz_\text{img}$ and $\vz_\text{text}$ %
which are expected to conform to UoS structures, where $\vz_\text{img}=\f(\vx)$ and $\vz_\text{text}=\g(\vt)$. 

For clarity and compactness in notation, we denote $\mZ_\text{img}\coloneqq[\vz_{\text{img},1},\cdots,\vz_{\text{img},N}]$ and $\mZ_\text{text}\coloneqq[\vz_{\text{text},1},\cdots,\vz_{\text{text},N}]$ %
as the learned representations of all images and all associated texts, respectively, and denote $\mC\coloneqq[\vc_1,\cdots,\vc_N]$ as the modality-shared self-expressive coefficient matrix.
Then, for fixed $\mZ_\text{img}$ and $\mZ_\text{text}$, problem~(\ref{eq:se_model_share_c}) is %
equivalently reformulated as follows:
\begin{align}
\label{eq:fix_Z_se_model_share_c_matrix}
\begin{aligned}
    \mathop{\min}_{\mC}~& \|\mZ_\text{img}-\mZ_\text{img}\mC\|_F^2+\|\mZ_\text{text}-\mZ_\text{text}\mC\|_F^2+r(\mC),\\
    \mathrm{s.t.}~&\Diag(\mC)=\boldsymbol{0}.
\end{aligned}
\end{align}
To handle the raw data from real-world, %
we thus learn simultaneously the embeddings $\mZ_\text{img}$ and $\mZ_\text{text}$ and the coefficient matrix $\mC$ by solving the following problem: 
\begin{align}
\label{eq:deep_se_model_share_c_matrix}
    \!\!\!\!\!\!\!\!\!\!\!\!\!\!\!&\mathop{\min}_{\mZ_\text{img},\mZ_\text{text},\mC}\|\mZ_\text{img}-\mZ_\text{img}\mC\|_F^2+\|\mZ_\text{text}-\mZ_\text{text}\mC\|_F^2+r(\mC),\nonumber \\
    ~~&\mathrm{s.t.}\Diag(\mC)=\boldsymbol{0}, %
    \|\mZ_\text{img}\|_F^2 =N, \|\mZ_\text{text}\|_F^2 =N, 
\end{align}
where the constraints %
are used to avoid the trivial solution.\footnote{In our implementation, we normalize %
$\{\vz_{\text{img},i},\vz_{\text{text},i}\}_{i=1}^N$ to have unit $\ell_2$ norm to avoid the trivial solution $\mZ_\text{img}=\mZ_\text{text}=\boldsymbol{0}, \mC=\boldsymbol{0}$.}  %

Unfortunately, such a joint optimization problem in (\ref{eq:deep_se_model_share_c_matrix}) suffers from a catastrophic representation collapse issue, analogous to the phenomenon %
observed in the %
deep self-expressive model~\cite{Haeffele:ICLR21}. 
In this paper, inspired by \cite{Meng:ICLR25}, we impose a maximal total coding rate regularization on each modality of the learned %
representations %
which are defined as follows:
\begin{align}
\label{eq:PRO_DSC_share_c_matrix}
\begin{aligned}
    \rho(\mZ_\text{img})\coloneqq&\log\det(\mI+\frac{d}{N\epsilon^2}\mZ_\text{img}\mZ_\text{img}^\top),\\
    \rho(\mZ_\text{text})\coloneqq&\log\det(\mI+\frac{d}{N\epsilon^2}\mZ_\text{text}\mZ_\text{text}^\top),
\end{aligned}
\end{align}
where $\epsilon > 0$ is a hyperparameter and %
$\rho(\cdot)$ measures the volume of the subspace spanned by the learned representations.\footnote{The volume of space is estimated by %
counting the number of $\epsilon$-balls that can be packed~\cite{Wright-Ma-HD2A}.} %

Putting these together, we propose a simple but principled approach for deep multi-modal subspace clustering, called \textbf{deep} \underline{\textbf{mo}}dality-sha\underline{\textbf{r}}ed \underline{\textbf{s}}elf-\underline{\textbf{e}}xpressive model (DeepMORSE):
\begin{align}
\label{eq:deep_se_model_share_c_matrix_reg}
\begin{aligned}
    \mathop{\min}_{\mZ_\text{img},\mZ_\text{text},\mC}~& \gamma\left(\|\mZ_\text{img}-\mZ_\text{img}\mC\|_F^2+\|\mZ_\text{text}-\mZ_\text{text}\mC\|_F^2\right)\\
    &\quad-\rho(\mZ_\text{img})-\rho(\mZ_\text{text})+r(\mC),\\
    \mathrm{s.t.}~&\Diag(\mC)=\boldsymbol{0}, %
    \|\mZ_\text{img}\|_F^2 =N, \|\mZ_\text{text}\|_F^2 =N, 
\end{aligned}
\end{align}
where $\gamma\in\mathbb{R}_+$ is a balancing hyperparameter. 

\myparagraph{Remark 1}
The incorporated regularization terms $-\rho(\mZ_\text{img})$ and $-\rho(\mZ_\text{text})$ are to maximize the volume of the spanned subspace of the learned representations. As analyzed in~\cite{Meng:ICLR25}, the total coding rate regularization can not only prevent the catastrophic representation collapse but also promote the learned representations conforming to a union of the modality-specific subspaces.
From an alternating optimization perspective, when $\{\mZ_\text{img}, \mZ_\text{text}\}$ are fixed, optimizing the shared coefficient matrix $\mC$ in DeepMORSE encourages modality-shared subspace-preserving structures; whereas when $\mC$ is fixed, optimizing $\{\mZ_\text{img}, \mZ_\text{text}\}$ in DeepMORSE refines the multimodal representations to better conform to a union of modality-specific subspace structures.

\subsection{Theoretical Justification}
\label{sec:theory}

\begin{proposition}
\label{proposition:1}
    Consider the sub-problem of %
    DeepMORSE with fixed representations $\mZ_\text{img}$ and $\mZ_\text{text}$, \ie,
    \begin{align}
    \begin{aligned}
    \label{eq:subprob}
     \!\!\!\vc_\text{share}^j\coloneqq\mathop{\arg\min}_{\vc^j} \|\vz_\text{img}^j-\mZ_\text{img}^{-j}\vc^j\|_2^2+\|\vz_\text{text}^j-\mZ_\text{text}^{-j}\vc^j\|_2^2,
    \end{aligned}
    \end{align}
    where $\mZ_\text{img}^{-j}$ and $\mZ_\text{text}^{-j}$ denote the representation matrices excluding the $j$-th column.
    Let $\vc_{\star}^j$ %
    be the modality-invariant self-expressive coefficients, \ie, 
    \begin{equation}
    \label{eq:c_res}
        \vz_\text{img}^j = \mZ_\text{img}^{-j}\vc_{\star}^j+\vdelta_\text{img}^j,\quad
        \vz_\text{text}^j = \mZ_\text{text}^{-j}\vc_{\star}^j+\vdelta_\text{text}^j,
    \end{equation}
    where $\vdelta_\text{img}^j,\vdelta_\text{text}^j$ denote the modality-specific deviation terms. Assume that $\mathbb{E}[\vdelta_\text{img}^j]=\mathbb{E}[\vdelta_\text{text}^j]=\boldsymbol{0}$, $\mathrm{Cov}[\vdelta_\text{img}^j]=\mathrm{Cov}[\vdelta_\text{text}^j]=\sigma^2\boldsymbol{I}_d$, then we have:
    
    a) $\vc_\text{share}^j$ is closer to $\vc_\star^j$ than both the modality-specific optimal solutions $\vc_\text{img}^j$ and $\vc_\text{text}^j$, \ie, 
    \begin{equation}
        \mathbb{E}\Big[\|\vc_\text{share}^j-\vc_\star^j\|_2^2\Big]\leq\mathop{\min}\Big(\mathbb{E}\Big[\|\vc_\text{img}^j-\vc_\star^j\|_2^2\Big],\mathbb{E}\Big[\|\vc_\text{text}^j-\vc_\star^j\|_2^2\Big]\Big).
    \end{equation}
     
    b) If $\vc_\star^j$ is subspace-preserving, then the self-expressive coefficients corresponding to %
    mismatched subspaces (which are indicated by $\mathcal{I}_-^j$) satisfy that:
    \begin{equation}
        \!\!\!\!\!\!\mathbb{E}\Big[\|\vc_{\text{share},\mathcal{I}_-^j}^j\|_2^2\Big]\leq\mathop{\min}\Big(\mathbb{E}\Big[\|\vc_{\text{img},\mathcal{I}_-^j}^j\|_2^2\Big],\mathbb{E}\Big[\|\vc_{\text{text},\mathcal{I}_-^j}^j\|_2^2\Big]\Big).
    \end{equation}
    
\end{proposition}

We now turn to the other side of the alternating optimization, where we optimize the modality-specific representations $\{\mZ_\text{img},\mZ_\text{text}\}$ when $\mC$ is fixed. 
We consider the following problem:
\begin{align}
\label{eq:fixedC_subprob}
\begin{aligned}
    \mathop{\arg\min}_{\mZ}~&-\log\det(\boldsymbol{I}+\alpha\mZ\mZ^\top)+\gamma\|\mZ-\mZ\mC\|_F^2,\\
    \mathrm{s.t.}~& \| \mZ \|_F^2 = N, %
\end{aligned}
\end{align}
where $\alpha=d/(N\epsilon^2)$. This formulation is exactly %
the regularized deep self-expressive model analyzed in~\cite{Meng:ICLR25}. It is easy to see that the two modality-specific representations $\{\mZ_\text{img},\mZ_\text{text}\}$ are non-collapsed under mild conditions as follows. %
\begin{proposition}[\cite{Meng:ICLR25}]
    Let $\mM\coloneqq(\mI-\mC_\text{share})(\mI-\mC_\text{share})^\top$, the optimal multimodal representations $\{\mZ_{\text{img},\star},\mZ_{\text{text},\star}\}$ learned by solving problem %
    (\ref{eq:deep_se_model_share_c_matrix_reg}) %
    are %
    non-collapsed, \ie, satisfying
    $$\operatorname{rank}(\mZ_{\text{img},\star})=\operatorname{rank}(\mZ_{\text{text},\star})=\min\{d,N\},$$ with the %
    singular values given by 
    \begin{align}
        \sigma_{\mZ_{\text{img},\star}}^{(i)}=\sigma_{\mZ_{\text{text},\star}}^{(i)}&=\sqrt{\frac{1}{\gamma\sigma_{\mM}^{(i)}+\nu_\star}-\frac{1}{\alpha}},
    \end{align}
    for $i=1,\cdots,\min\{d,N\}$, provided that $\gamma<(\alpha-\nu_\star)/\lambda_\text{max}(\mM)$, where $\nu_\star$ is the optimal dual variable.
\end{proposition}

\myparagraph{Remark 2}
The two propositions %
characterize a mutually reinforcing mechanism: the shared self-expressive coefficients suppress modality-specific noise %
towards being modality-specific subspace-preserving, which in turn guides the modality-specific representations toward a structured UoS distribution.  %
Then, the refined structured representations %
help learn %
shared self-expressive coefficients in subsequent iterations. 
 
Finally, from an implementation perspective, if a random mini-batch scheme is used to solve $\mC$, then there is an implicit regularization $\|\mC\|_F^2$ in the self-expressive model. %
Precisely, we have the following result. 

\begin{proposition}
\label{proposition:mini-batch-implicit-reg}
Let $\boldsymbol{\xi}=(\xi_1,\dots,\xi_N)^\top$ be a random vector consisting of $N$ i.i.d. Bernoulli random variables with
\begin{equation}
\xi_{i}=\left\{\begin{array}{ll}
1,     &  \text{with probability }{n_b}/{N},\\
0,     &  \text{with probability }1-{n_b}/{N},
\end{array}\right.
\end{equation}
where $n_b$ is the batch size. 
Then, using a stochastic mini-batch scheme to solve $\mC$ in the sub-problem of DeepMORSE essentially addresses the following problem~\footnote{%
The regularization terms not involving $\vc_j$ are ignored.}
\begin{align}
\label{eq:subprob}
     \mathop{\min}_{\mC}~& \|\mZ_{\text{img}}\Diag(\boldsymbol{\xi}) - \mZ_{\text{img}}\Diag(\boldsymbol{\xi})\mC\Diag(\boldsymbol{\xi})\|_F^2 \nonumber\\
     & + \|\mZ_{\text{text}}\Diag(\boldsymbol{\xi}) - \mZ_{\text{text}}\Diag(\boldsymbol{\xi})\mC\Diag(\boldsymbol{\xi})\|_F^2, \nonumber\\
    \mathrm{s.t.}~&\Diag(\mC)=\boldsymbol{0},%
\end{align}
which %
induces an implicit regularization $\frac{2(N-n_b)}{n_b}\|\mC\|_F^2$. 
\end{proposition}
This result reveals that the mini-batch training procedure of the self-expressive coefficients introduces an implicit %
$\|\mC\|_F^2$ regularization. %
This also explains the empirical observation in~\cite{Meng:ICLR25} that explicit regularization on the self-expressive coefficients has merely marginal effect---because the mini-batch optimization itself already induces %
an implicit regularizer.

\subsection{Implementations} 
\label{Sec:Implementations}

\myparagraph{Textual Counterpart Generation} 
In practical image clustering scenarios, off-the-shelf image-text pairs are seldom available.
Therefore, we leverage pretrained vision-language models (\eg, CLIP~\cite{Radford:ICML21-CLIP}) %
to generate a textual counterpart for each image. 
Formally, we denote $\vx\in\mathbb{R}^D$ as an image embedding and $\mathcal{D}=\{\boldsymbol{d}_1,\cdots,\boldsymbol{d}_{M}\}$ as a concept dictionary of text embeddings. Following prior work \cite{Li:ICML24-TAC}, we build the concept dictionary $\mathcal{D}$ %
by first applying $k$-means algorithm to the image embeddings and then, for each cluster center, selecting the top-ranked WordNet text embeddings based on the similarity between image embeddings and text embeddings. 

Given %
the concept dictionary $\mathcal{D}$, we then generate 
the textual counterparts for each image  %
that accurately capture the semantic content of the image while maintaining %
the clear geometric structure. 
To be specific, for each image embedding $\vx\in\mathbb{R}^D$, we %
seek its textual counterpart $\vt:=\sum_{i=1}^{M} \theta_i \boldsymbol{d}_i$ by solving the following cross-modal sparse coding problem:
\begin{align}
\label{eq:problem_for_text}
\begin{aligned}
    \mathop{\min}_{\boldsymbol{\theta}}\quad &\|\vx - \sum_{i=1}^{M} \theta_i \boldsymbol{d}_i\|_2^2,~~~
    \mathrm{s.t.} ~~~ & \|\boldsymbol{\theta}\|_0\leq s,
\end{aligned}
\end{align}
where $\|\boldsymbol{\theta}\|_0$ denotes the number of non-zero coefficients and $s$ is a positive integer to specify the sparsity. 
Note that the query comes from the image modality while the dictionary is constructed from the text modality. The describable correspondence %
of the textual counterpart $\vt$ is guaranteed %
by minimizing the Euclidean distance between $\vx$ and $\vt$, in which the sparsity constraint enforces that only the top-$s$ most semantically relevant atomic concepts are picked, %
so that the generated textual counterpart $\{\vt_i\}_{i=1}^N$ approximately %
reside on the \emph{subspace} spanned by a small set of atomic concepts %
in $\mathcal{D}$. 
By solving problem (\ref{eq:problem_for_text}) for all images $\mathcal{X}\coloneqq\{\vx_1,\cdots,\vx_N\}$, we obtain the %
textual counterparts $\mathcal{T} \coloneqq \{\vt_1,\cdots, \vt_N\}$, which accurately reflect the content of each image while approximately conforming to a UoS structure. 
In experiments, we adopt matching pursuit algorithm~\cite{Mallat:TSP93-MP} to solve problem~(\ref{eq:problem_for_text}).

\myparagraph{Reparameterization and Training Procedure} 
To solve problem (\ref{eq:deep_se_model_share_c_matrix_reg}) efficiently, we use MLPs to reparameterize %
$\f$ and $\g$ as $\f(~\cdot~;\boldsymbol{\phi}_I)$ and $\g(~\cdot~;\boldsymbol{\phi}_T)$,
where $\boldsymbol{\phi}_I$ and $\boldsymbol{\phi}_T$ denote the learnable parameters.
Accordingly, the modality-specific (\ie, image and text) representations are obtained by:
\begin{align}
\label{eq:mlp_forward}
\begin{aligned}
    \vz_\text{img}=&\f(\vx;\boldsymbol{\phi}_I)/\|\f(\vx;\boldsymbol{\phi}_I)\|_2,\\
    \vz_\text{text}=&\g(\vt;\boldsymbol{\phi}_T)/\|\g(\vt;\boldsymbol{\phi}_T)\|_2,
\end{aligned}
\end{align}
where the unit $\ell_2$-norm %
is to satisfy the normalization constraint in %
problem (\ref{eq:deep_se_model_share_c_matrix}).
Since the number of expressive coefficients $\{c_{ij}\}_{i,j=1}^N$ grows quadratically with $N$, we follow \cite{Zhang:CVPR21-SENet,Ding:ICCV23,Meng:ICLR25} to reparameterize the coefficients with a properly designed two-branch network. 
Specifically, we first compute a modality-mixed representation $\vz_\text{mix}\coloneqq\frac{1}{2}(\vz_\text{img}+\vz_\text{text})$ and then %
define the expressive coefficients %
as follows:\footnote{After the Sinkhorn-Knopp projection~\cite{Caron:NIPS20-SwAV,Ding:ICML22,Ding:ICCV23}, we eliminate the diagonal elements of $\mC$ to satisfy the constraint, and multiply the projected coefficients element-wise by $\mathrm{sign}(\mZ_\text{mix}^\top\mZ_\text{mix})$ to restore the sign, as the self-expressive coefficients are not required to be non-negative. Thus, we refer to it as a ``signed'' Sinkhorn-Knopp projection.}
\begin{align}
\label{eq:compute_C}
    \mC=\mathcal{P}(\mZ_\text{mix}^\top \mZ_\text{mix}),
\end{align}
where $\mathcal{P}(\cdot)$ is %
a signed Sinkhorn-Knopp projection~\cite{Cuturi:NIPS13-sinkhorn}.

\myparagraph{Remark 3}
We clarify %
the key differences between our reparameterization of $\mC$ and that of SENet~\cite{Zhang:CVPR21-SENet}. 
First, although SENet is also trained with SGD, %
its self-expressive coefficients are computed with respect to the dictionary which consists of the full training data. 
As a result, there is no implicit regularization induced by its training procedure and the computational cost scales with $N$. 
In our DeepMORSE, 
the training procedure via SGD in mini-batch mode effectively induces an implicit regularization $\ell_2^2$, similar to the implicit regularization analyzed in~\cite{Chen:CVPR20}. 
Second, SENet introduces a learnable soft-thresholding operation to obtain a sparse self-expressive matrix, while DeepMORSE adopts ``signed'' Sinkhorn-Knopp projection to obtain the self-expressive matrix. Empirically, we observe that the ``signed'' Sinkhorn-Knopp projection yields better clustering performance compared to the soft-thresholding operator. 
Third, %
while the forward computation $\mC=\mathcal{P}(\mZ^\top\mZ)$ is %
similar to~\cite{Ding:ICCV23}, the interpretation is fundamentally different: the output of Sinkhorn-Knopp projection in \cite{Ding:ICCV23} is doubly stochastic and %
used as the affinity matrix; %
whereas the output of the ``signed'' Sinkhorn-Knopp projection\footnote{The Sinkhorn-Knopp projection has also been employed in %
manifold-constrained hyper-connections (mHC)~\cite{Xie:arxiv25-mhc}, where it projects residual connection matrices onto the Birkhoff polytope for stable signal propagation. While the context differs, we conjecture that the Sinkhorn-Knopp projection serves to %
numerically stabilizing the optimization.} in DeepMORSE is no longer doubly stochastic and serves as self-expressive coefficients to capture the underlying subspace structure.

Equipped with the reparameterization for $\mZ_\text{img}$, $\mZ_\text{text}$ and $\mC$, %
rather than directly solving %
$\mZ_\text{img}$, $\mZ_\text{text}$ and $\mC$ from problem (\ref{eq:deep_se_model_share_c_matrix_reg}), 
we instead train the networks by jointly updating %
the parameters $\{\boldsymbol{\phi}_I,\boldsymbol{\phi}_T\}$ via Stochastic Gradient Descent (SGD) in mini-batch mode. %
Specifically, for each iteration, we pick a mini-batch of $n_b$ image representations at random from $\mathcal{X}$ (denoted as  $\mZ^\mathcal{B}_\text{img}\in\mathbb{R}^{d\times n_b}$), and accordingly pick $n_b$ counterpart text representations from $\mathcal{T}$ (denoted as $\mZ^\mathcal{B}_\text{text} \in\mathbb{R}^{d\times n_b}$). 
Then, for each mini-batch, the loss function is computed by: 
\begin{align}
\label{eq:loss_func}
\begin{aligned}
    \mathcal{L}\coloneqq &\gamma\left(\|\mZ^\mathcal{B}_\text{img}-\mZ^\mathcal{B}_\text{img}\mC^\mathcal{B}\|_F^2+\|\mZ^\mathcal{B}_\text{text}-\mZ^\mathcal{B}_\text{text}\mC^\mathcal{B}\|_F^2\right)\\
    &\quad-\rho(\mZ^\mathcal{B}_\text{img})-\rho(\mZ^\mathcal{B}_\text{text}),
\end{aligned}
\end{align}
where the term $r(\cdot)$ is ignored due to the implicit regularization %
induced by SGD in a mini-batch mode %
(Proposition~\ref{proposition:mini-batch-implicit-reg}).

\myparagraph{Evaluation on Test Data}
After training, the clustering results are obtained via spectral clustering~\cite{Shi:TPAMI00} on the affinity matrix %
directly induced by $|\mZ_\text{img}^\top\mZ_\text{img}|$.
Here, we use %
only the image modality to find the clustering of the test data, %
without further textual counterpart generation. This makes the inductive clustering process simpler and %
straightforward.\footnote{This strategy also follows \cite{Cai:AAAI23-SIC,Li:ICML24-TAC}, ensuring a fair comparison.} 
Moreover, the structured multimodal representations $\mZ_\text{img}$ and $\mZ_\text{text}$ learned by DeepMORSE can be directly applied to downstream tasks such as zero-shot classification and image retrieval.

\begin{figure}[!t]
\centering
\resizebox{\linewidth}{!}{
\begin{minipage}{\linewidth}
\begin{algorithm}[H]
\caption{Deep Modality-Shared Self-Expressive Model (DeepMORSE)}
\label{alg:algorithm}
\begin{algorithmic}[1] %
\linespread{1.1} %
\item[\textbf{Input:}] Image embeddings $\mathcal{X}$, dictionary of text embeddings $\mathcal{D}$, hyperparameters $s,\gamma,\epsilon,n_b$, number of iterations $T$, learning rate $\eta$
\item[\textbf{Textual counterpart generation:}] Generate textual counterparts $\mathcal{T}=\{\vt_i\}_{i=1}^N$ by solving problem~(\ref{eq:problem_for_text})

\item[\textbf{Initialization:}] %
Initialize network parameters $\boldsymbol \phi_I, \boldsymbol \phi_T$

\For{$t=1,\dots,T$} 
    \State\textit{\# Forward propagation}
    \State Randomly select a mini-batch of data $\{\vx_i,\vt_i\}_{i\in\mathcal{B}}$
    \State Compute representations $\mZ^\mathcal{B}_\text{img},\mZ^\mathcal{B}_\text{text}$ by (\ref{eq:mlp_forward})

    \State Compute %
    self-expressive matrix $\mC^\mathcal{B}$ by (\ref{eq:compute_C})
    \State Compute loss $\mathcal{L}$ by (\ref{eq:loss_func})
    
    \State\textit{\# Backward propagation}
    \State Compute gradient $\nabla_{\boldsymbol \phi_I} \coloneqq \frac{\partial \mathcal{L}}{\partial {\boldsymbol \phi_I}} $, $\nabla_{\boldsymbol \phi_T} \coloneqq \frac{\partial \mathcal{L}}{\partial {\boldsymbol \phi_T}} $

    \State Update $\boldsymbol \phi_I,\boldsymbol \phi_T$ by SGD with learning rate $\eta$

\EndFor
\linespread{1.1} %
\item[\textbf{Test:}]  
\State Compute representations of test data $\mZ_\text{img}$ by (\ref{eq:mlp_forward})
\State Apply spectral clustering on $|\mZ_\text{img}^\top\mZ_\text{img}|$. %
\end{algorithmic}
\end{algorithm}
\end{minipage}
}
\end{figure}

For clarity, we illustrate our DeepMORSE in Figure~\ref{fig:fig2} and summarize the training %
procedure in Algorithm~\ref{alg:algorithm}. %

\section{Experiments}
\label{Sec:Experiments}

We first describe the experimental setup %
in Section~\ref{sec:Experimental_Setup},
then report the clustering performance comparison in Section~\ref{Sec:clustering_result} and show extensive evaluations in Section~\ref{sec:Evaluations_on_Representations}. 

\subsection{Experimental Setup}
\label{sec:Experimental_Setup}
The reparameterized transformations $\f(~\cdot~;\phi_I)$ and $\g(~\cdot~;\phi_T)$ are implemented as two-layer MLPs, each consisting of two fully-connected layers interleaved with batch normalization and a ReLU activation, with hidden dimension 512 and output dimension 128 by default. Nevertheless, 
for the datasets that include more than 128 categories, we increase the hidden dimensions and output dimensions of the model and adjust $\gamma$ proportionally %
according to the output dimension, which is based on the sufficient condition to avoid representation collapse~\cite{Meng:ICLR25}. We fix the hyperparameters $s$ and $\epsilon$ to $s=5$ and $\epsilon^2=0.1$ on all datasets. 
DeepMORSE is trained with a learning rate of $\eta=10^{-4}$ and a batch size of $n_b=1024$. We evaluate clustering performance using accuracy (ACC) and normalized mutual information (NMI), and report the average results over five random seeds. 
Please refer to Appendix \ref{sec:Experimental_Details} for more experimental details and our principled approach for hyper-parameters configuration.

\begin{table*}[tbp]
  \centering
  \caption{\textbf{Image Clustering Performance Comparison.} The best results are in \textbf{bold} and the second best results are \underline{underlined}. The method marked with $\dagger$ is based on our reproduction, due to the distinction of architectures. The methods marked with $*$ are training-free approaches.}
  \resizebox{\textwidth}{!}{
    \begin{tabular}{lcccccccccccccc}
    \toprule
    \rowcolor{myGray} &  &  & \multicolumn{2}{c}{CIFAR-10} & \multicolumn{2}{c}{CIFAR-20} & \multicolumn{2}{c}{Dogs-15} & \multicolumn{2}{c}{DTD-47} & \multicolumn{2}{c}{UCF-101} & \multicolumn{2}{c}{ImgNet-1k} \\
    \rowcolor{myGray} & \multirow{-2}{*}{Backbone} & \multirow{-2}{*}{Venue}  & ACC   & NMI   & ACC   & NMI   & ACC   & NMI   & ACC   & NMI   & ACC   & NMI  & ACC   & NMI \\
    \midrule
    \multicolumn{15}{l}{\color{gray}\textit{Without Textual Information}}\\
    \quad SCAN~\cite{Van:ECCV2020-SCAN} & ResNet-18 & ECCV'20 & 88.3 & 79.7 & 50.7 & 48.6 & 59.3 & 61.2 & 46.4 & 59.4 & 61.1 & 79.7 & 39.9 & - \\
    \quad GCC~\cite{Zhong:ICCV21-GCC} & ResNet-18 & ICCV'21 & 85.6 & 76.4 & 47.2 & 47.2  & 52.6 & 49.0 & - & - & - & - & - & - \\
    \quad NNM~\cite{Dang:CVPR21-NNM} & ResNet-18 & CVPR'21 & 83.7 & 73.7 & 45.9 & 48.0 & 58.6 & 60.4 & - & - & - & - & - & - \\
    \quad CoKe~\cite{Qian:CVPR22-CoKe} &  ResNet-18 & CVPR'22 & 85.7 & 76.6 & 49.7 & 49.1 & - & - & - & - & - & - & - & - \\
    \quad SeCu~\cite{Qian:ICCV23-SeCu} &  ResNet-18 & ICCV'23 & \textbf{93.0} & \textbf{86.1} & 55.2 & 55.1 & - & - & - & - & - & - & - & - \\
    \quad CLIP+$k$-means~\cite{Radford:ICML21-CLIP} & ViT-B/32 & ICML'21 & 74.2 & 70.3 & 45.5 & 49.9 & 38.1 & 39.8 & 42.6 & 57.3 & 58.2 & 79.5 & 38.9 & 72.3 \\ 
    \quad CLIP+SC~\cite{Radford:ICML21-CLIP} & ViT-B/32 & ICML'21 & 69.0 & 65.6 & 45.1 & 47.7 & 41.9 & 42.5 & 49.8 & 61.4 & 64.2 & 82.2 & 41.3 & 72.5 \\
    \quad PRO-DSC$^\dagger$~\cite{Meng:ICLR25} & ViT-B/32 & ICLR'25 & 86.7 & 79.1 & 58.8 & 62.1 & 32.3 & 25.8 & 48.1 & 57.0 & 68.0 & \underline{83.6}& - & - \\
    \midrule
    \multicolumn{15}{l}{\color{gray}\textit{With Textual Information}}\\
    \quad SIC~\cite{Cai:AAAI23-SIC} & ViT-B/32 & AAAI'23  & 92.6 & 84.7 & 58.3  & 59.3  & 69.7  & 69.0  & 45.9  & 59.6  & 61.9  & 81.0 & 47.0 & 77.2 \\
    \quad MCA~\cite{Qiu:AAAI24} & ViT-B/32 & AAAI'24 & \underline{92.7} & \underline{84.9} & 59.7 & 59.8 & 74.9 & 73.3 & - & - & - & - & - & -\\
    \quad TAC~\cite{Li:ICML24-TAC} & ViT-B/32 & ICML'24 & 91.9  & 83.3  & 60.7  & 61.1  & \underline{83.0}  & 80.6  & 50.1  & \underline{62.1}  & \underline{68.7}  & 82.3  & \underline{58.2} & \underline{79.9} \\
    \quad GradNorm$^*$~\cite{Peng:ICCV25} & ViT-B/32 & ICCV'25  & 91.1 & 82.6 & 60.6 & 61.3 & 81.2 & \underline{81.0} & \underline{50.9} & 61.3 & 62.7 & 82.9 & 52.6 & 79.2 \\
    \quad CAE$^*$~\cite{Zhu:AAAI26} & ViT-B/32 & AAAI'26  & 90.9 & 81.7 & \underline{60.9} & \underline{62.8} & 77.5 & 77.9 & 46.7 & \underline{62.1} & 63.5 & 82.8 & 53.1 & 79.3 \\
    \quad DeepMORSE & ViT-B/32 & Ours & $92.3{\scriptstyle\pm 0.1}$  & $84.1{\scriptstyle\pm 0.2}$  & $\textbf{63.3}{\scriptstyle\pm 0.6}$ & $\textbf{63.1}{\scriptstyle\pm 0.5}$ & $\textbf{87.9}{\scriptstyle\pm 1.1}$ & $\textbf{83.8}{\scriptstyle\pm 1.4}$ & $\textbf{54.7}{\scriptstyle\pm 1.0}$ & $\textbf{64.2}{\scriptstyle\pm 0.3}$ & $\textbf{71.9}{\scriptstyle\pm 0.8}$ & $\textbf{85.6}{\scriptstyle\pm 0.4}$ & $\textbf{60.8}{\scriptstyle\pm 0.2}$ & $\textbf{80.9}{\scriptstyle\pm 0.2}$ \\
    \bottomrule
    \end{tabular}%
    }
  \label{tab:tab1}%
\end{table*}%
\vskip 0.1 in

\begin{figure*}[!t]
    \centering
    \parbox{\linewidth}{\centering
    \includegraphics[trim=50pt 40pt 50pt 0pt, clip,width=0.2\linewidth]{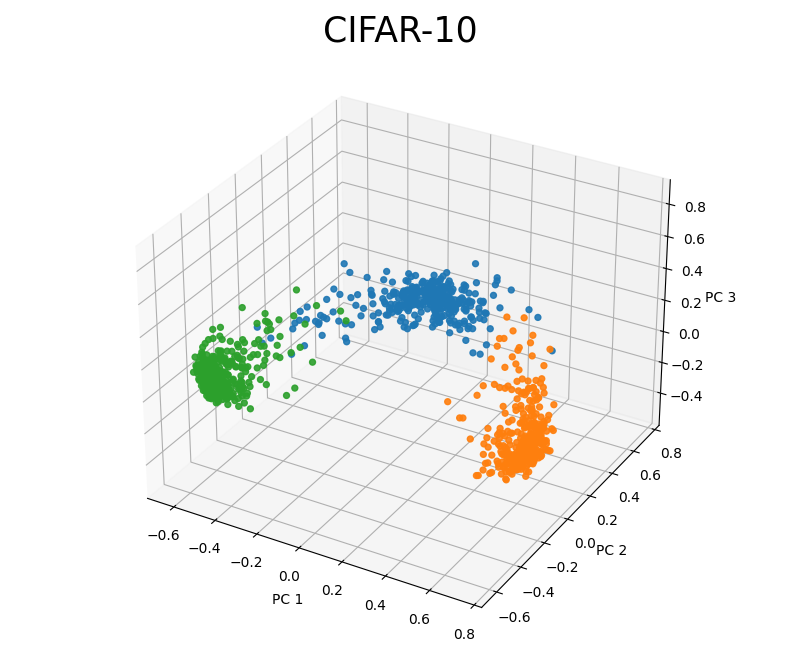}\hfill
    \includegraphics[trim=50pt 40pt 50pt 0pt, clip,width=0.2\linewidth]{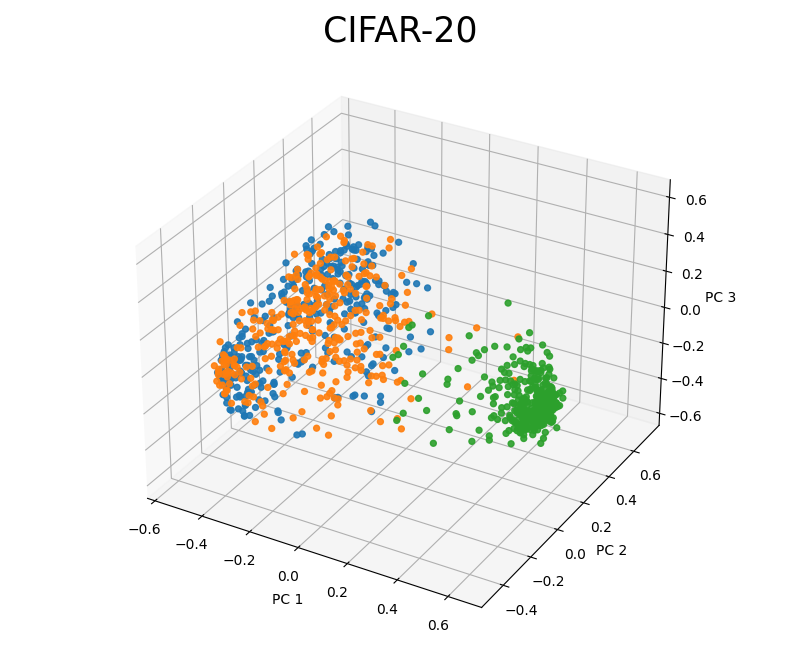}\hfill
    \includegraphics[trim=50pt 40pt 50pt 0pt, clip,width=0.2\linewidth]{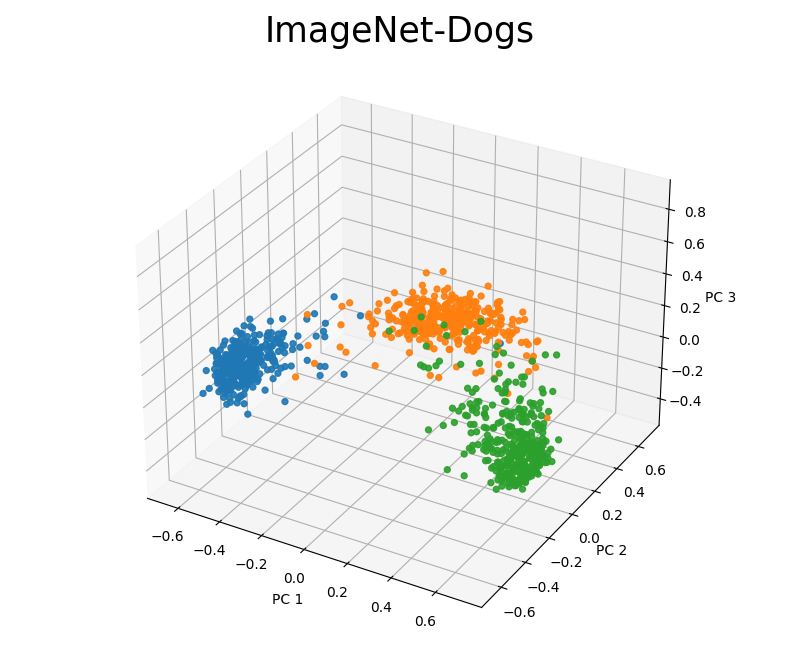}\hfill
    \includegraphics[trim=50pt 40pt 50pt 0pt, clip,width=0.2\linewidth]{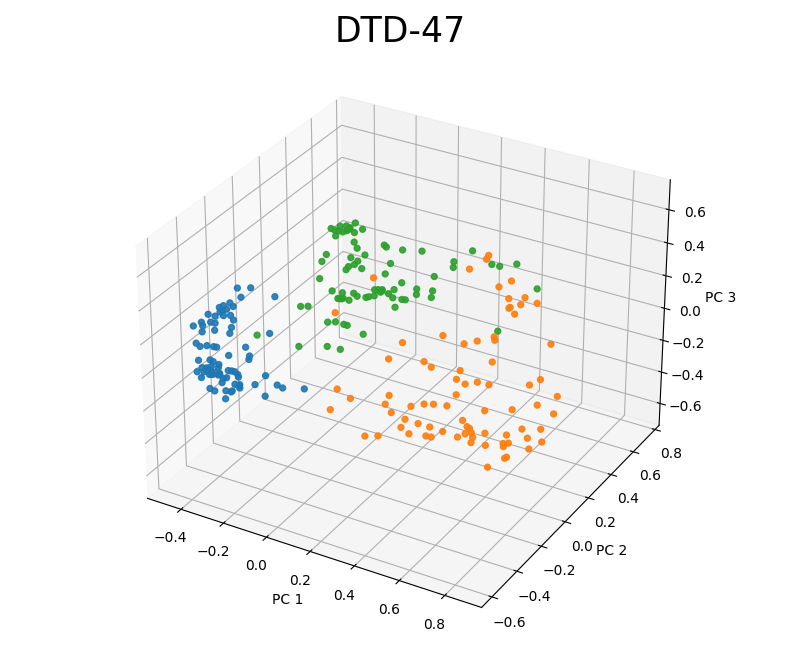}\hfill
    \includegraphics[trim=50pt 40pt 50pt 0pt, clip,width=0.2\linewidth]{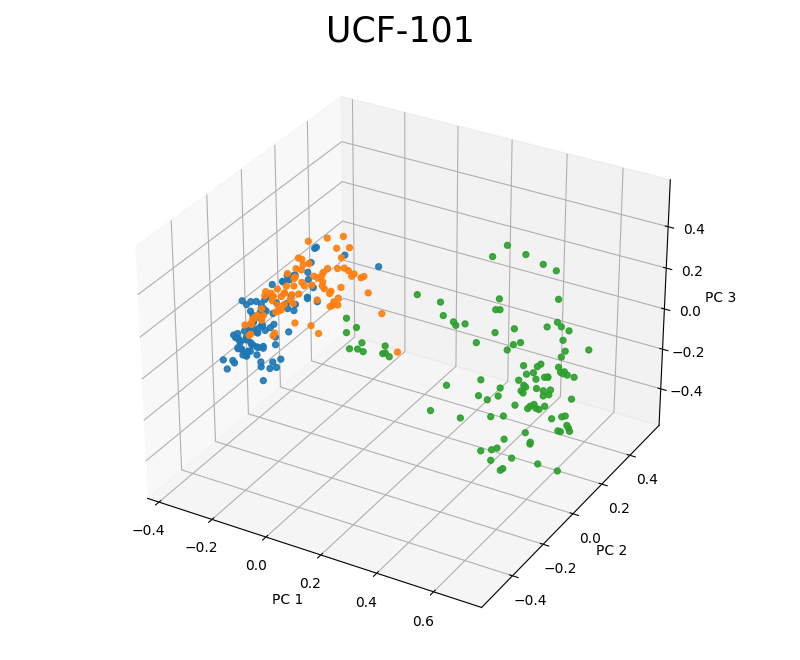}\\[4pt]
    \includegraphics[trim=50pt 40pt 50pt 40pt, clip,width=0.2\linewidth]{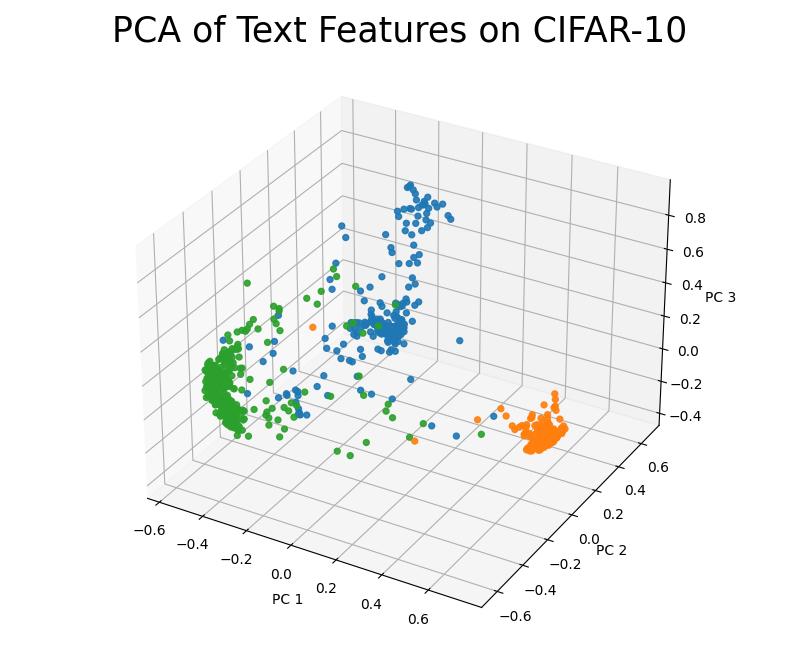}\hfill
    \includegraphics[trim=50pt 40pt 50pt 40pt, clip,width=0.2\linewidth]{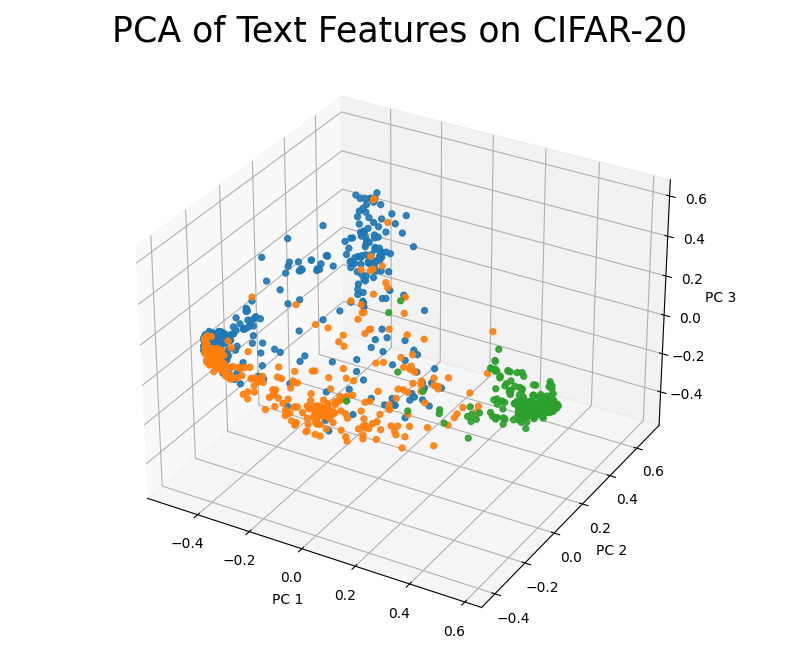}\hfill
    \includegraphics[trim=50pt 40pt 50pt 40pt, clip,width=0.2\linewidth]{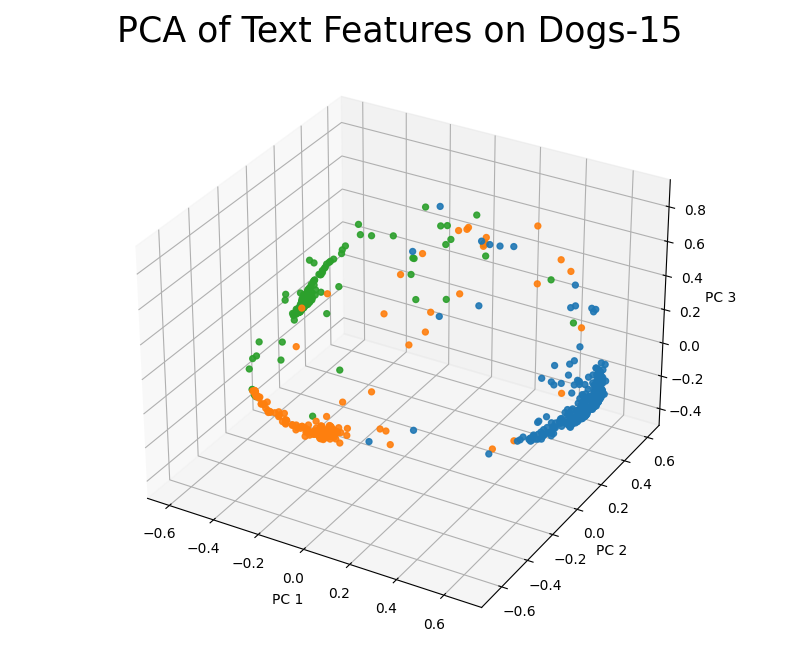}\hfill
    \includegraphics[trim=50pt 40pt 50pt 40pt, clip,width=0.2\linewidth]{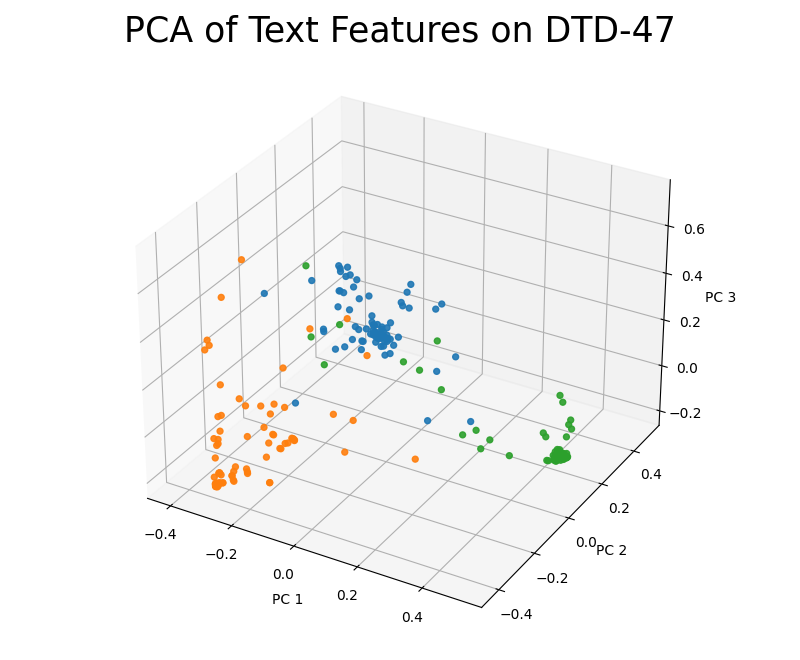}\hfill
    \includegraphics[trim=50pt 40pt 50pt 40pt, clip,width=0.2\linewidth]{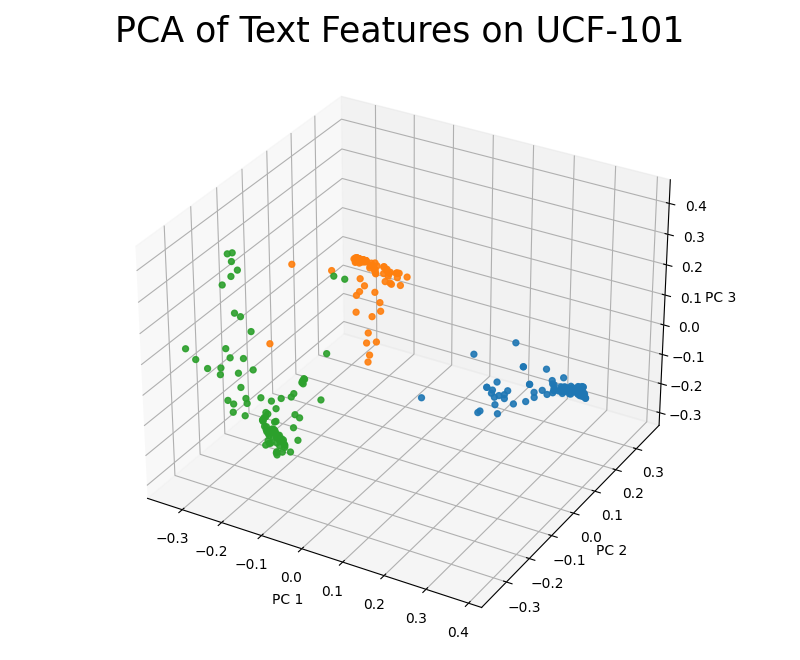}%
    }%
    
    \vspace{6pt}
    \parbox{\linewidth}{\footnotesize\normalfont\centering\textbf{(a)} TAC image (top) and text (bottom) representations}
    \vspace{0.15in}
    
    \parbox{\linewidth}{\centering
    \includegraphics[trim=50pt 40pt 50pt 0pt, clip,width=0.2\linewidth]{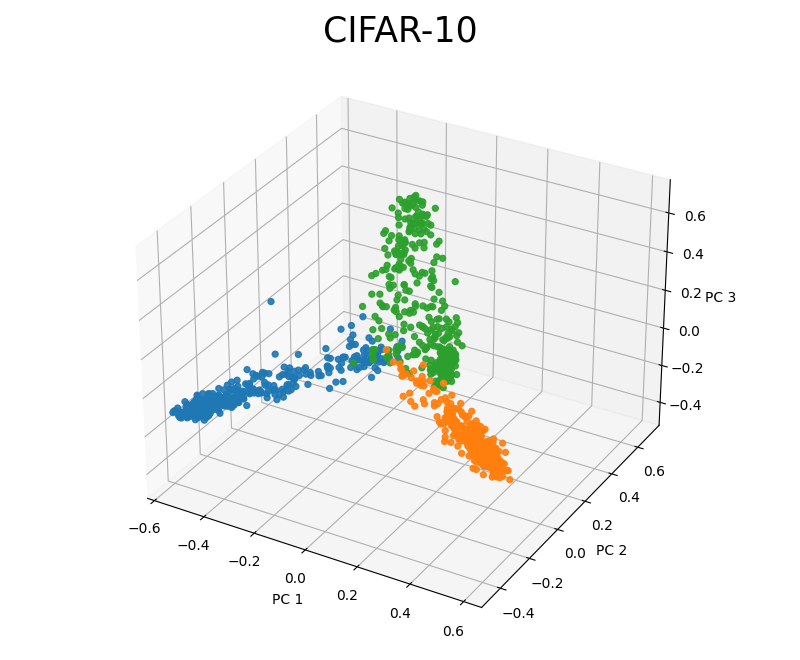}\hfill
    \includegraphics[trim=50pt 40pt 50pt 0pt, clip,width=0.2\linewidth]{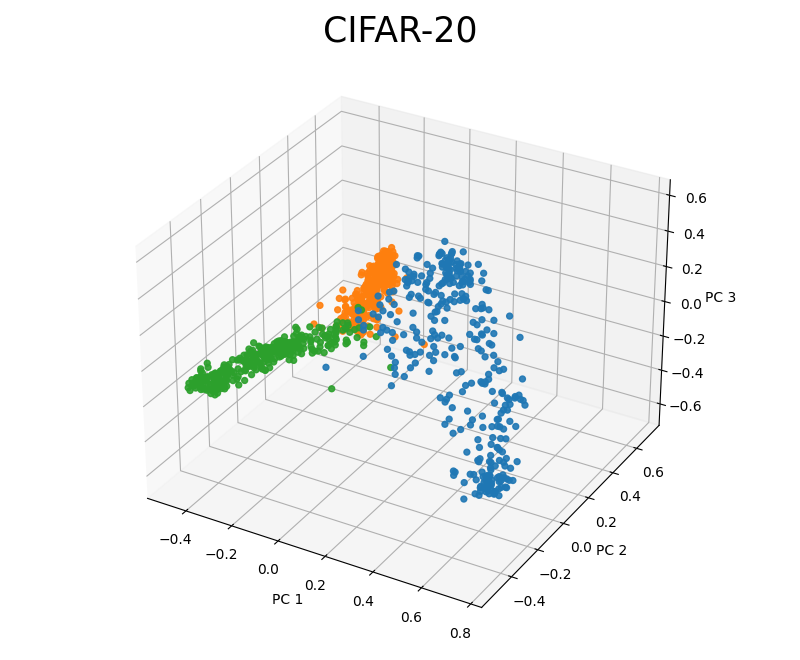}\hfill
    \includegraphics[trim=50pt 40pt 50pt 0pt, clip,width=0.2\linewidth]{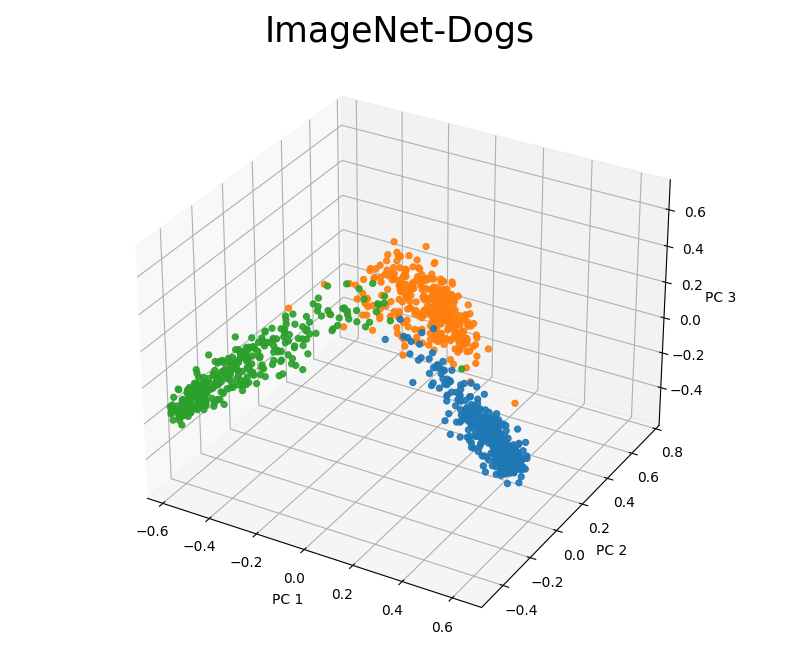}\hfill
    \includegraphics[trim=50pt 40pt 50pt 0pt, clip,width=0.2\linewidth]{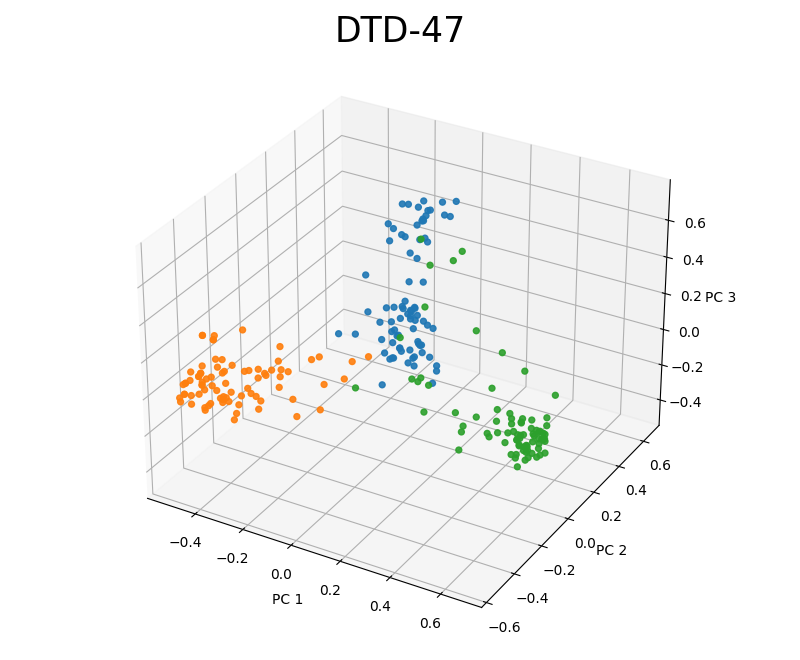}\hfill
    \includegraphics[trim=50pt 40pt 50pt 0pt, clip,width=0.2\linewidth]{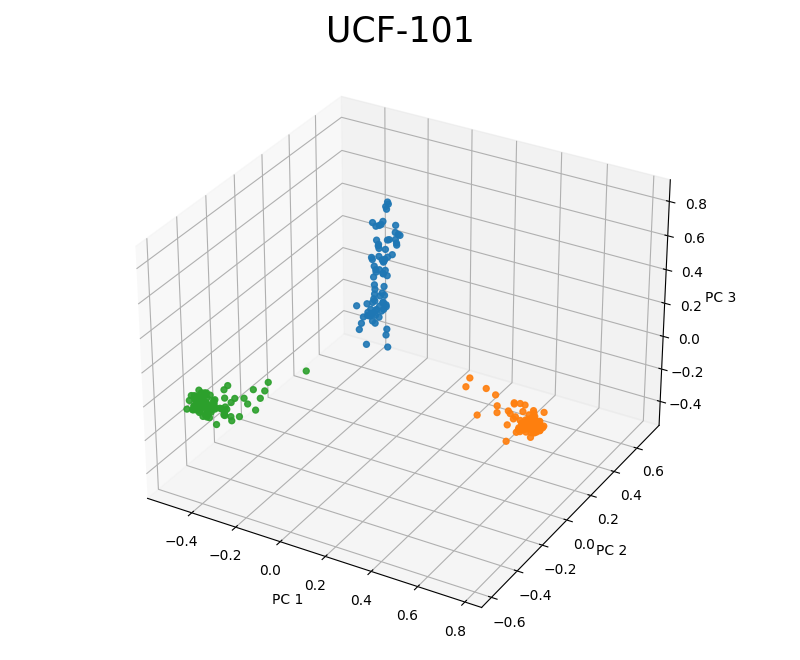}\\[4pt]
    \includegraphics[trim=50pt 40pt 50pt 40pt, clip,width=0.2\linewidth]{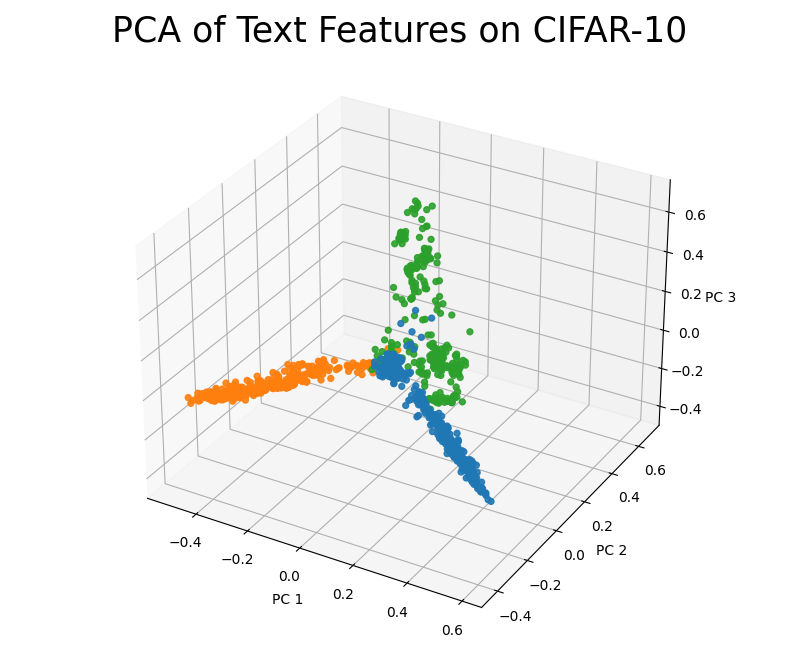}\hfill
    \includegraphics[trim=50pt 40pt 50pt 40pt, clip,width=0.2\linewidth]{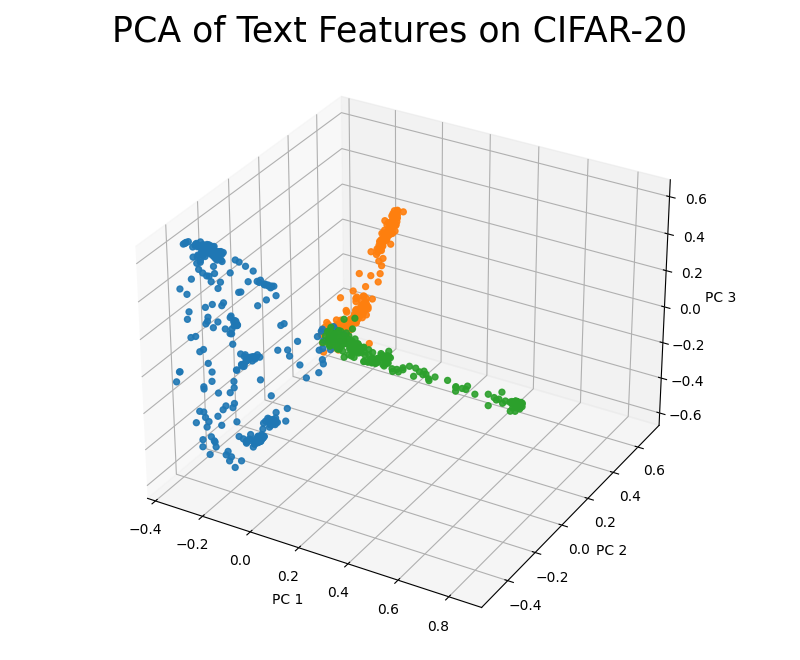}\hfill
    \includegraphics[trim=50pt 40pt 50pt 40pt, clip,width=0.2\linewidth]{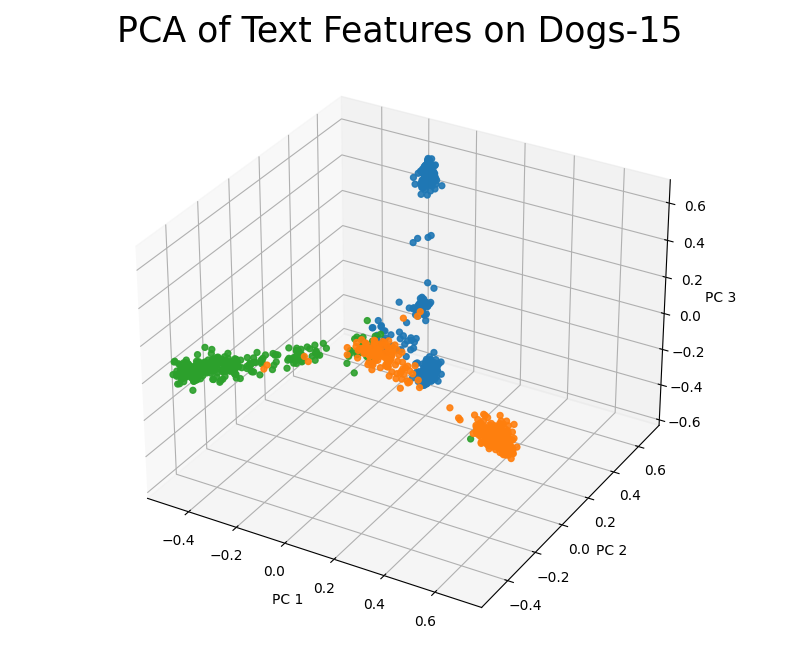}\hfill
    \includegraphics[trim=50pt 40pt 50pt 40pt, clip,width=0.2\linewidth]{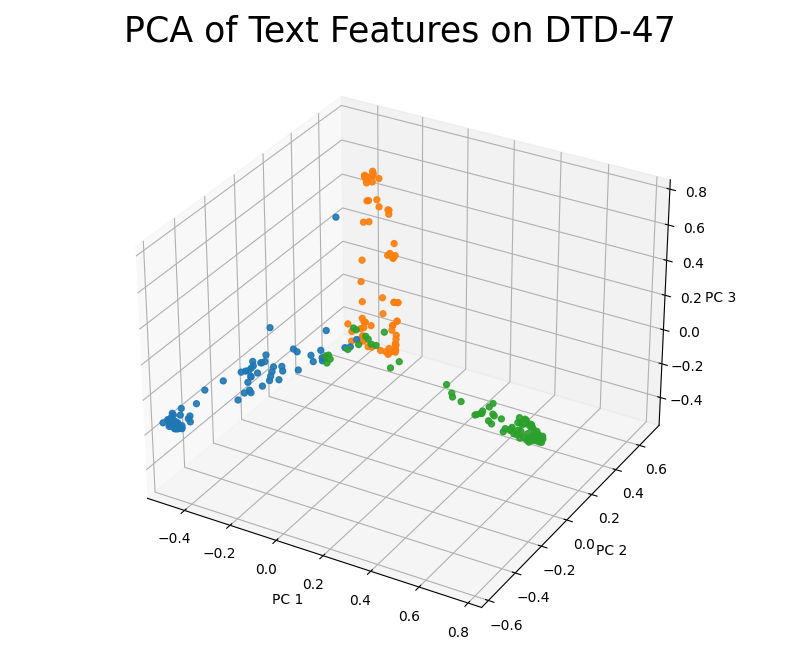}\hfill
    \includegraphics[trim=50pt 40pt 50pt 50pt, clip,width=0.2\linewidth]{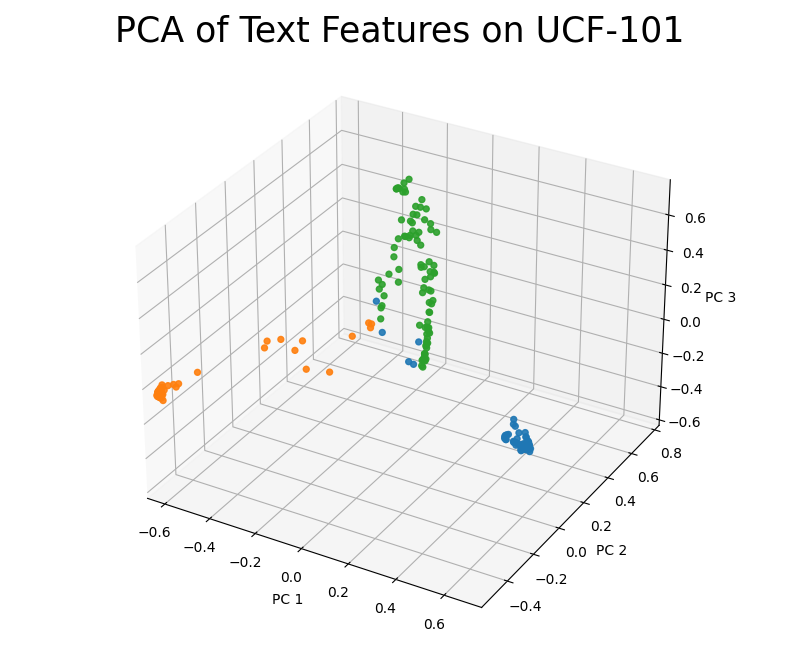}%
    }%
    
    \vspace{6pt}
    \parbox{\linewidth}{\footnotesize\normalfont\centering\textbf{(b)} DeepMORSE image (top) and text (bottom) representations}
    
\caption{\textbf{Visualization of image and text representations via PCA.}}
\label{fig:PCA}
\end{figure*}

\subsection{Main Results}
\label{Sec:clustering_result}

\begin{figure}[!t]
    \centering
    \includegraphics[trim=10pt 10pt 10pt 0pt, clip,width=0.49\linewidth]{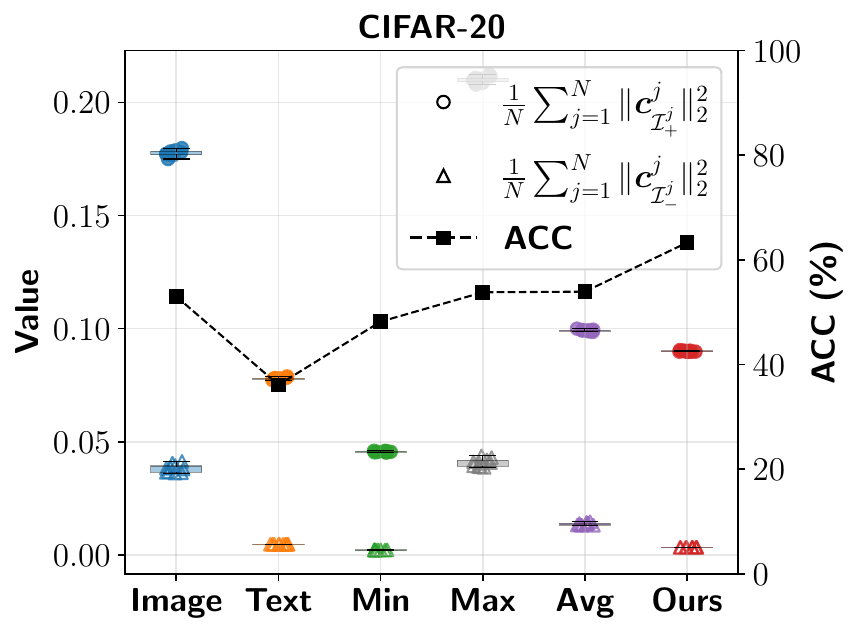}\hfill
    \includegraphics[trim=10pt 10pt 10pt 0pt, clip,width=0.49\linewidth]{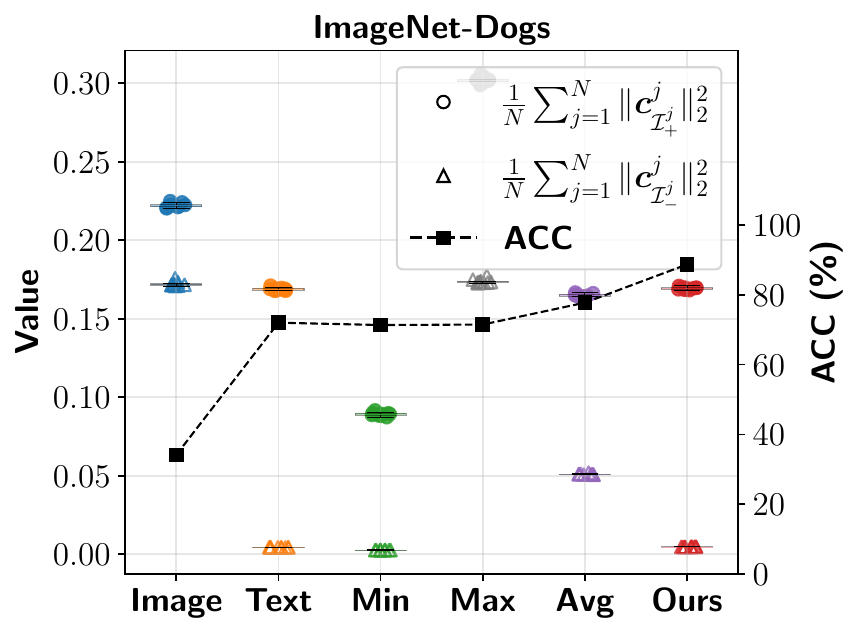}\\
    \vskip 0.1in
    \includegraphics[trim=10pt 0pt 10pt 0pt, clip,width=0.49\linewidth]{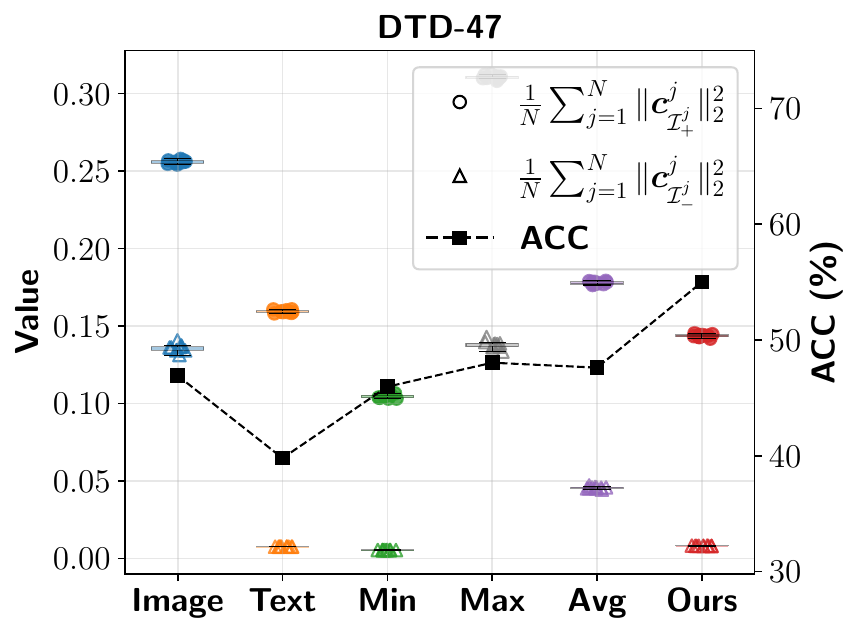}\hfill
    \includegraphics[trim=10pt 0pt 10pt 0pt, clip,width=0.49\linewidth]{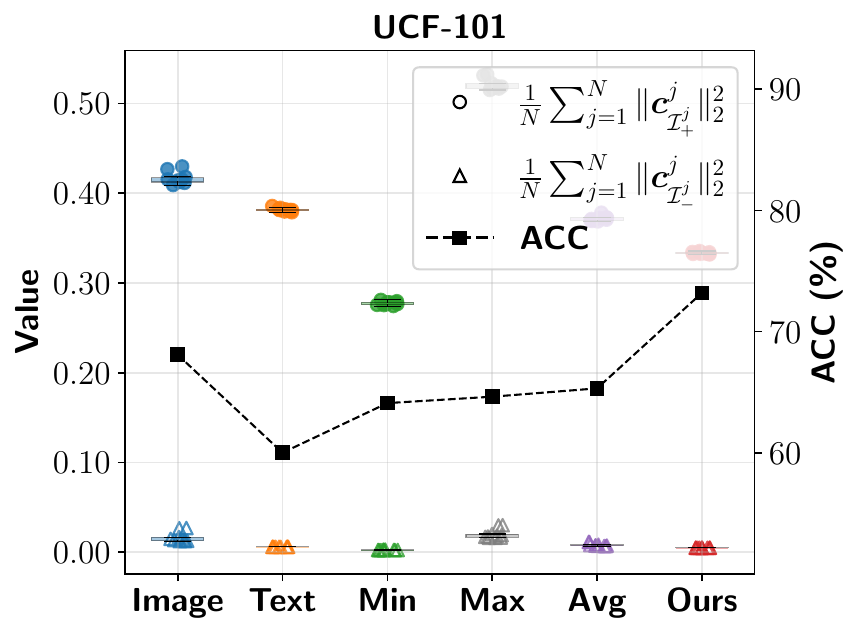}
\caption{\textbf{Illustration of the average intra/inter-class self-expressive coefficients along with clustering accuracy.}}
\label{fig:supp_C}
\end{figure}

\myparagraph{Clustering Performance}
We conduct experiments on six image clustering benchmarks, including CIFAR-10~\cite{Krizhevsky:2009-CIFAR}, CIFAR-20~\cite{Krizhevsky:2009-CIFAR},
ImageNet-Dogs~\cite{Chang:ICCV17-DAC}, DTD-47~\cite{Cimpoi:CVPR14}, UCF-101~\cite{Soomro:arxiv12}, and ImageNet-1k~\cite{Deng:CVPR09}.
We compare DeepMORSE with unimodal deep clustering baselines (\eg, GCC~\cite{Zhong:ICCV21-GCC}, NNM~\cite{Dang:CVPR21-NNM}, SCAN~\cite{Van:ECCV2020-SCAN}, CoKe~\cite{Qian:CVPR22-CoKe}, SeCu~\cite{Qian:ICCV23-SeCu}, PRO-DSC~\cite{Meng:ICLR25}) as well as multimodal methods (\eg, SIC~\cite{Cai:AAAI23-SIC}, TAC~\cite{Li:ICML24-TAC}, MCA~\cite{Qiu:AAAI24}, GradNorm~\cite{Peng:ICCV25}, CAE~\cite{Zhu:AAAI26}).\footnote{Since that GradNorm and CAE are training-free methods, we compare DeepMORSE to TAC for most cases.}
For CLIP$+$$k$-means and CLIP$+$Spectral Clustering (SC), we compute the image embeddings with pretrained CLIP image encoder and then directly run $k$-means or spectral clustering algorithm on the obtained embeddings, respectively.

As shown in Table~\ref{tab:tab1}, DeepMORSE achieves state-of-the-art clustering performance except for CIFAR-10. 
Specifically, it improves clustering accuracy by $3.8\%$, $3.2\%$, and $4.9\%$ on DTD-47, UCF-101, and ImageNet-Dogs, respectively.
In addition, since PRO-DSC is a deep self-expressive model for unimodal image clustering, the improvements of DeepMORSE over PRO-DSC provide clear evidence that leveraging textual information can %
effectively %
enhance %
image clustering.

\myparagraph{Visualization via PCA}
To visualize the structure of representations learned by TAC~\cite{Li:ICML24-TAC} and our DeepMORSE, we %
display the 
dimension reduction results via PCA of the learned representations for each modality on five benchmark datasets.
Note that PCA preserves the global structure of the representations as it performs the %
dimension reduction by %
a \emph{linear} projection. %
As shown in Figure~\ref{fig:PCA}, the representations learned by our DeepMORSE (in rows 3 and 4) approximately form a union of orthogonal low-dimensional subspaces, where the representations from each class are concentrated within a distinct subspace. For TAC (in rows 1 and 2), the projected embeddings exhibit a variety of structures that cannot be easily approximated by a union of centroids, subspaces, or manifolds.

\myparagraph{DeepMORSE Suppresses Inter-Class Coefficients} 
Proposition~\ref{proposition:1} demonstrates that the modality-shared self-expressive coefficients provably suppress inter-class self-expressive coefficients compared to modality-specific solutions. We now empirically verify this statement.
Specifically, we compare the shared self-expressive matrix learned by our DeepMORSE against those learned separately on each modality, in terms of the per-sample intra-class coefficients $\frac{1}{N}\sum_{j=1}^N\|\vc^j_{\mathcal{I}_{+}^j}\|_2^2$, the per-sample inter-class coefficients $\frac{1}{N}\sum_{j=1}^N\|\vc^j_{\mathcal{I}_{-}^j}\|_2^2$, and the resulting clustering accuracy.
We also include ``post-fusion'' baselines obtained by integrating the modality-specific coefficients via element-wise minimum, maximum, and average operations, respectively.
As shown in Figure~\ref{fig:supp_C}, %
the inter-class coefficients produced by our DeepMORSE %
are the smallest among all methods. This validates the theoretical finding that modality-shared self-expression suppresses inter-class noise. 
At the same time, its intra-class coefficients remain comparably strong, which %
account for the satisfactory clustering accuracy achieved by our DeepMORSE.

\begin{figure}[tbp]
    \centering
    \includegraphics[trim=0pt 0pt 0pt 0pt, clip,width=0.48\linewidth]{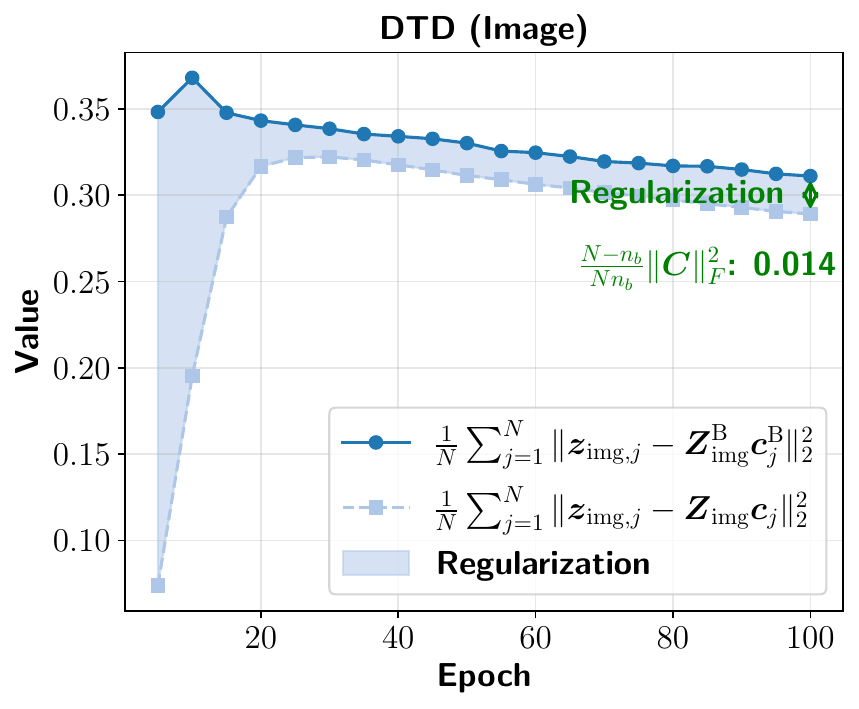}\hfill
    \includegraphics[trim=0pt 0pt 0pt 0pt, clip,width=0.48\linewidth]{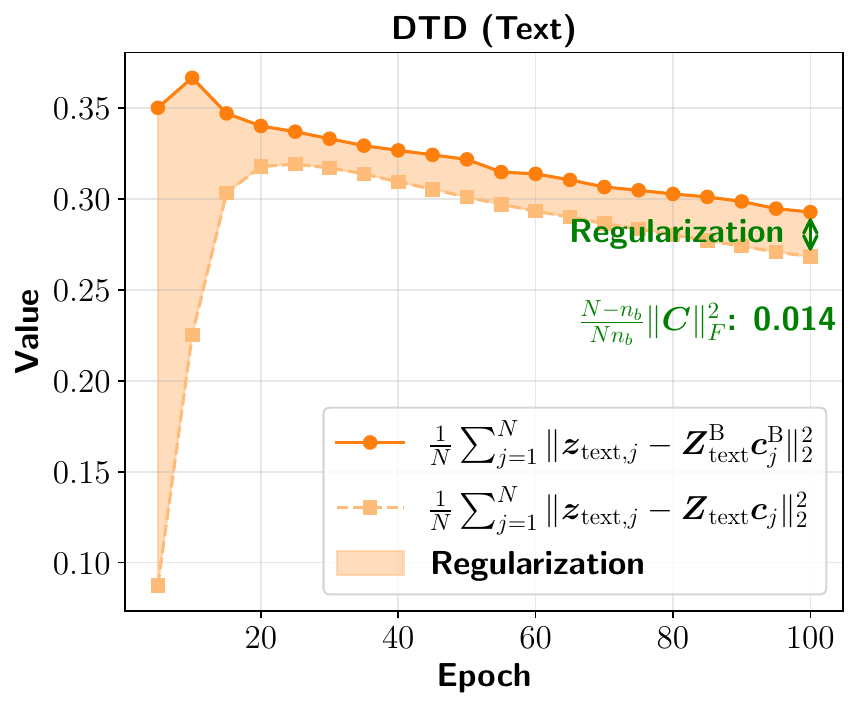}\\
    \includegraphics[trim=0pt 0pt 0pt 0pt, clip,width=0.48\linewidth]{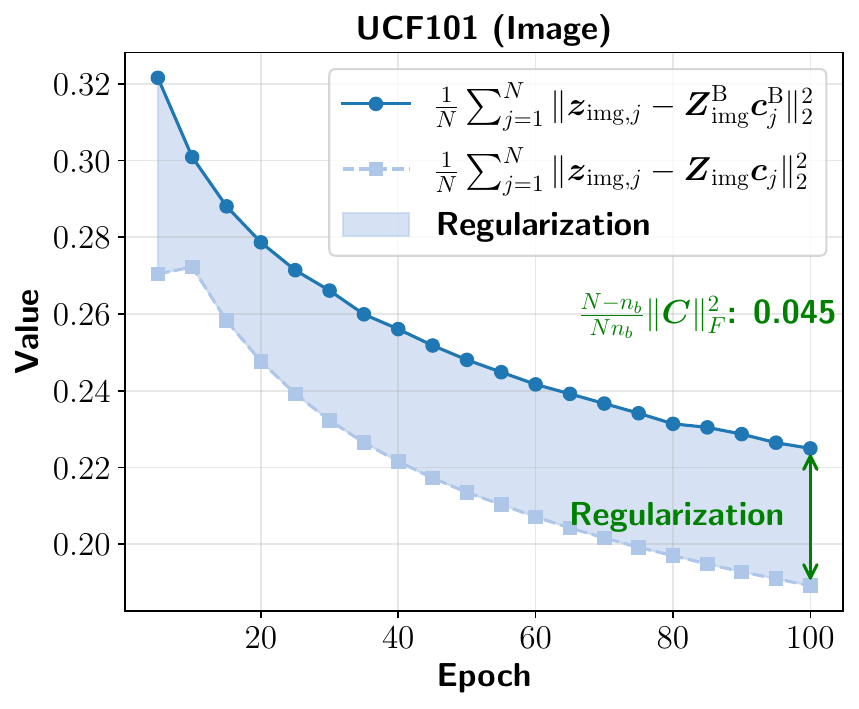}\hfill
    \includegraphics[trim=0pt 0pt 0pt 0pt, clip,width=0.48\linewidth]{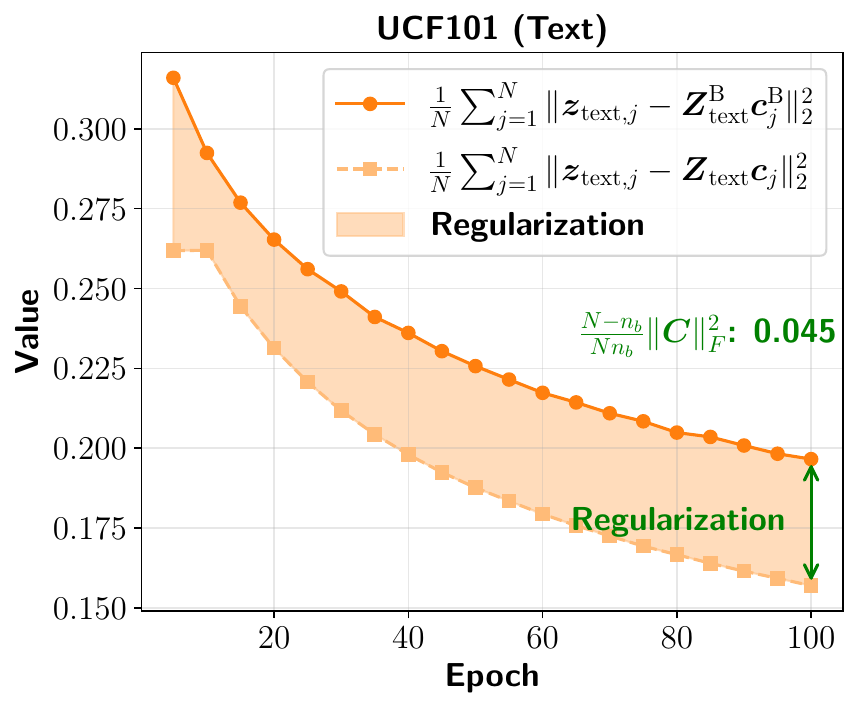}
\caption{\textbf{Validating Implicit regularization %
induced by SGD in mini-batch.}}
\label{fig:propo2}
\vskip 0.12 in
\end{figure}

\myparagraph{Implicit Regularization in Mini-batch Training}
To empirically verify the implicit Frobenius-norm regularization induced by the mini-batch training scheme as pointed out in Proposition~\ref{proposition:mini-batch-implicit-reg}, we compare the self-expression loss computed during mini-batch training against that computed using the full training data.
During standard training with batch size $n_b$, the per-sample self-expressive loss for each batch is computed as $\frac{1}{n_b}\|\mZ^\mathcal{B}-\mZ^\mathcal{B}\mC^\mathcal{B}\|_F^2$ for whichever the image or text modality. If we take an average over all mini-batches within each epoch, we have %
the per-sample mini-batch loss %
$\frac{1}{N}\sum_{j=1}^N\|\vz_{\text{img},j}-\mZ_{\text{img}}^\mathcal{B}\vc_j^\mathcal{B}\|_2^2$ and $\frac{1}{N}\sum_{j=1}^N\|\vz_{\text{text},j}-\mZ_{\text{text}}^\mathcal{B}\vc_j^\mathcal{B}\|_2^2$. %
To obtain the self-expressive loss without using mini-batch subsampling scheme, for each iteration, we forward the entire training set through the network to obtain the full representations $\mZ_{\text{img}},\mZ_{\text{text}}\in\mathbb{R}^{d\times N}$, and then compute the complete self-expressive matrix $\mC$ via (\ref{eq:mlp_forward}) and (\ref{eq:compute_C}). The corresponding full-data loss %
$\frac{1}{N}\sum_{j=1}^N\|\vz_{\text{img},j}-\mZ_{\text{img}}\vc_j\|_2^2$ and $\frac{1}{N}\sum_{j=1}^N\|\vz_{\text{text},j}-\mZ_{\text{text}}\vc_j\|_2^2$ are then obtained. %
According to Proposition~\ref{proposition:mini-batch-implicit-reg}, the per-sample loss gap between the mini-batch scheme and the full-data scheme is expected to be %
$\frac{N-n_b}{Nn_b}\|\mC\|_F^2$ for each modality---which is due to the implicit regularization.
We conduct training experiments on datasets DTD-47 and UCF-101 with batch size $n_b=1024$ and display the loss curves during training period in Figure~\ref{fig:propo2}. 
As can be observed, %
the quantitatively %
measured gap between the loss functions %
tends to closely match the theoretical prediction in Proposition~\ref{proposition:mini-batch-implicit-reg} on both datasets when the training is terminated. %
This %
confirms the implicit regularization due to %
the mini-batch training in our DeepMORSE.

\begin{table*}[!t]
  \centering
  \caption{\textbf{Ablation studies.}}
  \resizebox{0.9\linewidth}{!}{
    \begin{tabular}{P{1.2cm}|P{0.6cm}P{0.6cm}P{0.8cm}P{0.8cm}|P{1.3cm}P{1.3cm}P{1.7cm}P{1.3cm}P{1.3cm}}
    \toprule
    \rowcolor{myGray}
    Text & \multicolumn{4}{c|}{Loss function} & \multicolumn{5}{c}{Clustering performance (ACC: \%)} \\
    \rowcolor{myGray} generation & $\mathcal{L}_\text{img}^\text{expr}$   & $\mathcal{L}_\text{text}^\text{expr}$  & $\rho(\mZ_\text{img})$   & $\rho(\mZ_\text{text})$ & CIFAR-10 & CIFAR-20 & Dogs-15 & DTD-47 & UCF-101 \\
    \midrule
    - & \ding{51} &       & \ding{51} &        & 86.7& 58.8& 32.3& 48.1  & 68.0  \\
    \cmidrule(lr){1-1}\cmidrule(lr){2-5}\cmidrule(lr){6-10}
    \multirow{3}{*}{\rotatebox[origin=c]{-30}{Neighbor}}   &  & \ding{51} &  & \ding{51}  & 69.8 & 34.0&  72.2 &  40.0 & 60.1  \\
       & \ding{51} & \ding{51} &  &   & 32.8 & 9.6 & 36.3 & 28.2  & 44.0  \\
       & \ding{51} & \ding{51} & \ding{51} & \ding{51}  & \underline{92.0} & \underline{62.4}& \textbf{88.5}& \underline{54.3}  & \textbf{72.2}  \\
    \cmidrule(lr){1-1}\cmidrule(lr){2-5}\cmidrule(lr){6-10}  &       & \ding{51} &       & \ding{51}  & 71.2 & 34.5 & 72.9 & 40.2  & 60.7  \\
     & \ding{51} & \ding{51} &       &        & 33.0 & 9.4 & 17.2 & 27.0  & 47.8  \\
    \multirow{-3}{*}{\rotatebox[origin=c]{-30}{Sparse}} & \ding{51} & \ding{51} & \ding{51} & \ding{51}  & \textbf{92.3} & \textbf{63.3} & \underline{87.9} & \textbf{54.7}  & \underline{71.9}  \\
    \bottomrule
    \end{tabular}
    }
  \label{tab:ablation}%
\end{table*}%

\begin{figure}[tbp]
    \centering
    \includegraphics[trim=0pt 0pt 0pt 0pt, clip,width=0.48\linewidth]{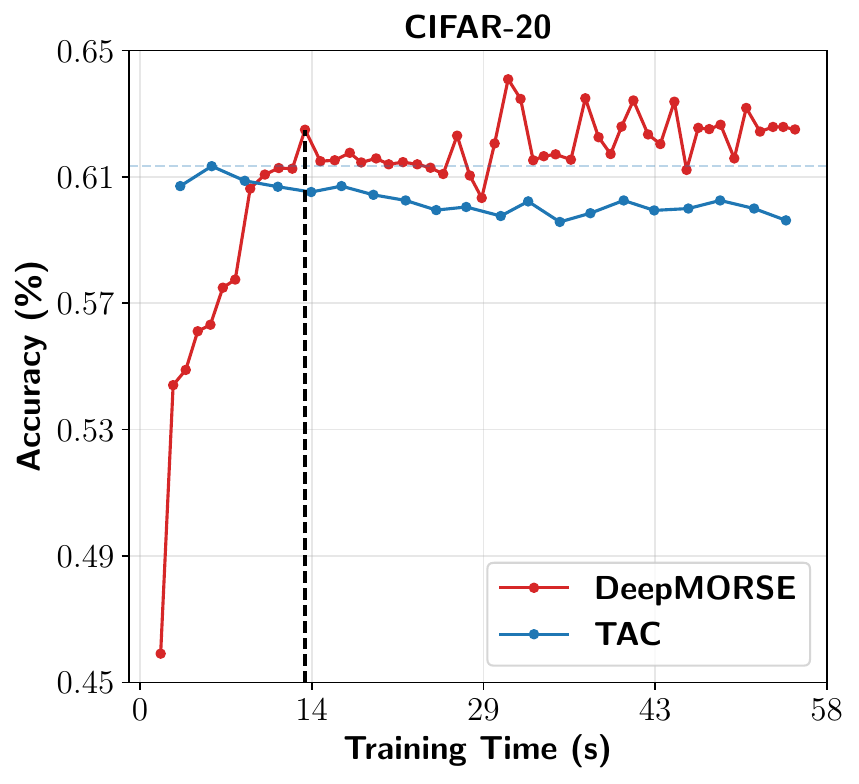}\hfill
    \includegraphics[trim=0pt 0pt 0pt 0pt, clip,width=0.48\linewidth]{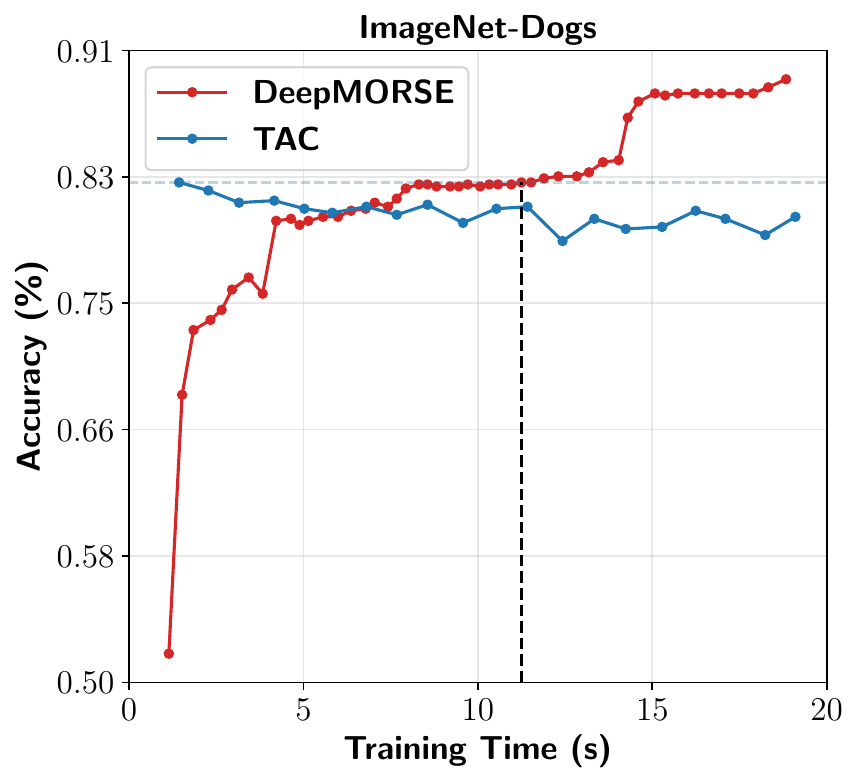}\\
    \vskip 0.04in
    \includegraphics[trim=0pt 0pt 0pt 0pt, clip,width=0.48\linewidth]{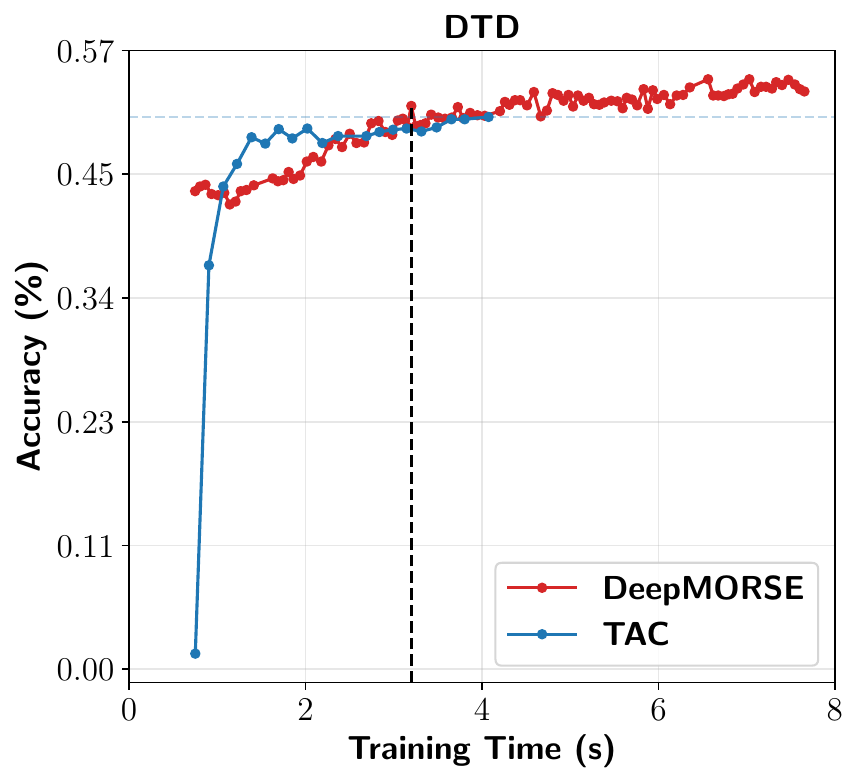}\hfill
    \includegraphics[trim=0pt 0pt 0pt 0pt, clip,width=0.48\linewidth]{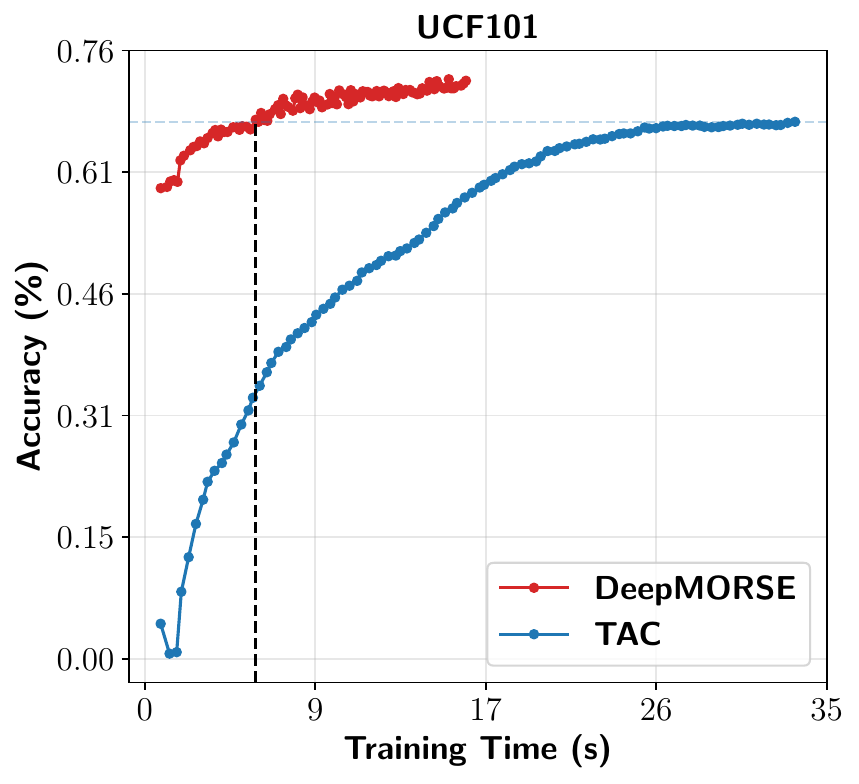}
\caption{\textbf{Clustering accuracy of DeepMORSE and TAC at different training time.} Red: DeepMORSE; blue: TAC.}
    \vskip 0.09in
\label{fig:time}
\end{figure}

\myparagraph{Ablation Studies}
To evaluate the contribution of each component in our DeepMORSE, we conduct a set of ablation studies and report the results in Table~\ref{tab:ablation}.
First, we evaluate unimodal variants by removing the loss from the other modality.\footnote{Unimodal self-expressive matrices are defined as $\mC_\text{img}\coloneqq\mathcal{P}(\mZ_\text{img}^\top\mZ_\text{img})$ and $\mC_\text{text}\coloneqq\mathcal{P}(\mZ_\text{text}^\top\mZ_\text{text})$, with $\gamma$ selected from $\{0.5,1,5,10,50,100,150,300,500\}$.}
As shown in first, second and fifth rows, %
the clustering performance degrades sharply when only using a single modality.
Removing the representation regularization terms further causes severe degeneration (see third row and sixth row), %
which is consistent with the collapse problem reported in~\cite{Haeffele:ICLR21,Meng:ICLR25}.
Finally, when replacing our textual counterpart generation method with neighborhood-based retrieval~\cite{Li:ICML24-TAC} (fourth row), %
the results still surpass those of \cite{Li:ICML24-TAC}, highlighting the robustness of our modality-shared self-expression.

\myparagraph{Complexity Analysis}
Note that by the commutative property $\log\det(\boldsymbol{I}+\Z\Z^\top)=\log\det(\boldsymbol{I}+\Z^\top\Z)$ (see \cite{Ma:TPAMI07}), we can reduce the computation of $\log\det(\cdot)$ from an $N \times N$ matrix to a $d \times d$ matrix. 
This decouples the complexity from the number of samples $N$ and yields a time and space complexity of $\mathcal{O}(d^3)$, where typically $d \ll N$.
Since that the computation of $\log\det(\cdot)$ is the bottleneck in optimizing our DeepMORSE, this property makes the loss evaluation substantially more efficient.
To evaluate the computational efficiency of DeepMORSE, we measure the time cost on a single NVIDIA RTX 5090 GPU and Intel(R) Xeon(R) Platinum 8470Q CPU.
Specifically, we plot the clustering accuracy as a function of training time for both TAC and DeepMORSE in Figure \ref{fig:time}. We also mark the time at which DeepMORSE reaches TAC’s best performance.
As shown, to reach the best clustering accuracy of TAC, DeepMORSE requires only about half of TAC's training time (and less than 20 seconds for each dataset).

\begin{figure}[tbp]
    \centering
    \includegraphics[trim=0pt 15pt 20pt 42pt, clip,width=0.48\linewidth]{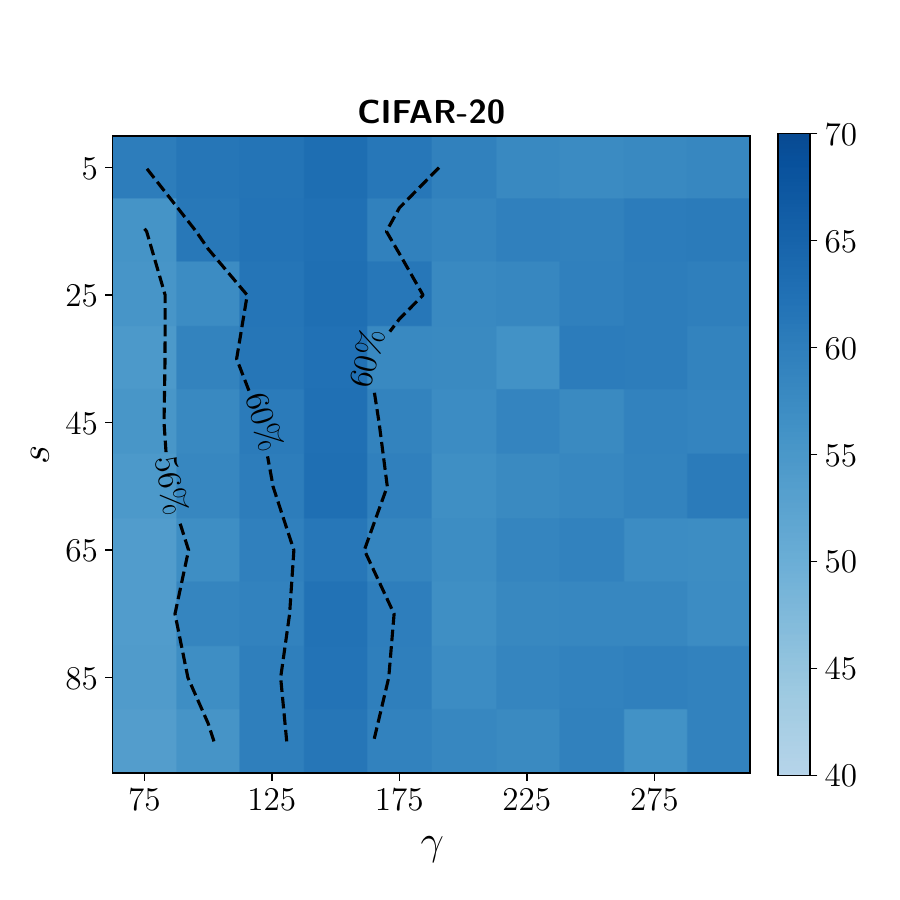}\hfill
    \includegraphics[trim=0pt 15pt 20pt 42pt, clip,width=0.48\linewidth]{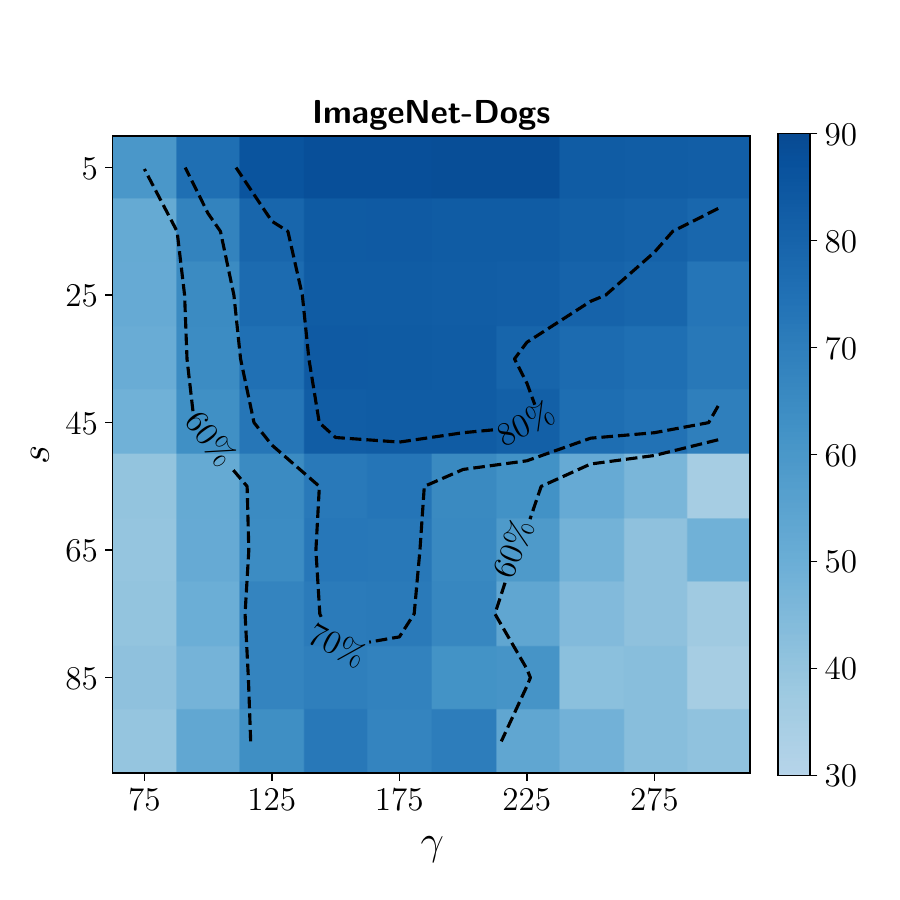}\\
    \vskip 0.01in
    \includegraphics[trim=0pt 15pt 20pt 32pt, clip,width=0.48\linewidth]{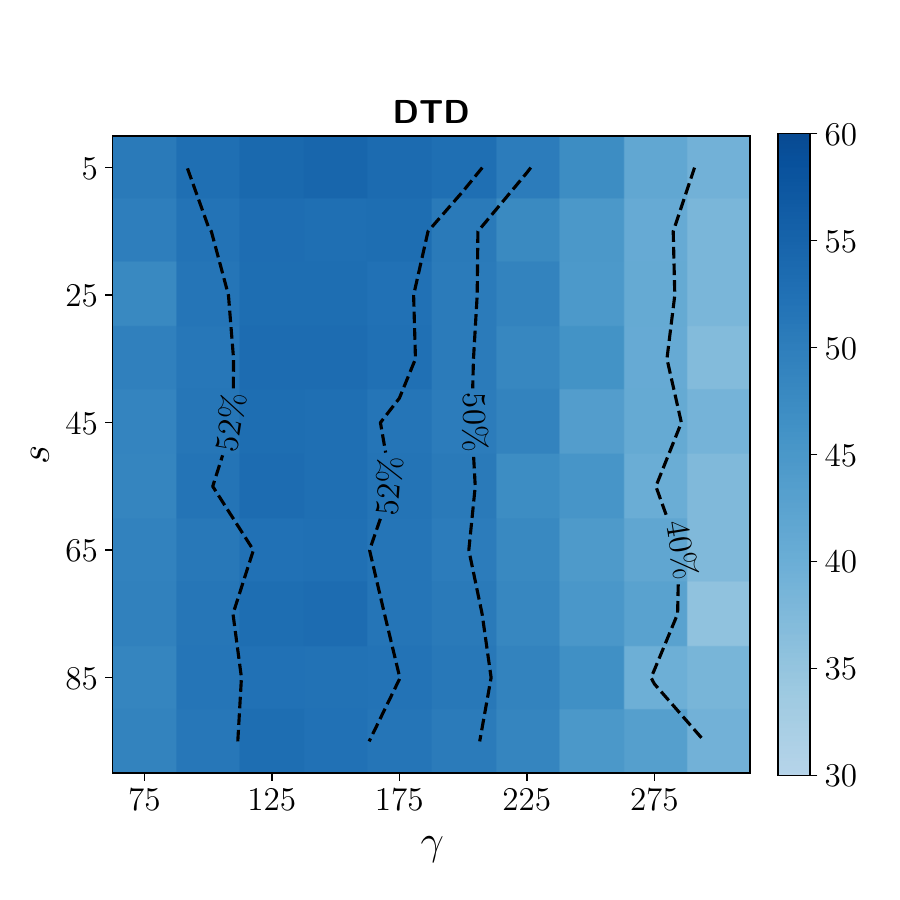}\hfill
    \includegraphics[trim=0pt 15pt 20pt 32pt, clip,width=0.48\linewidth]{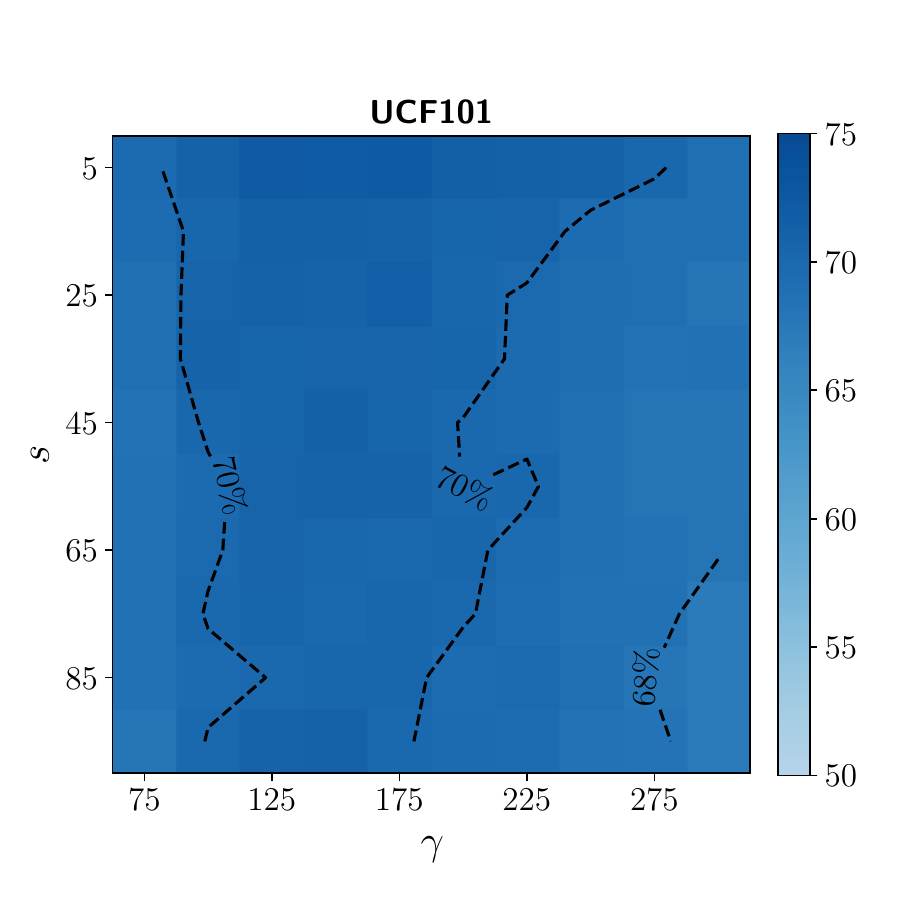}
\caption{\textbf{Clustering accuracy of DeepMORSE %
with respect to $\gamma$ and $s$.}}
\label{fig:sensitivity}
\end{figure}

\myparagraph{Sensitivity of Hyperparameters}
To discover the influence of hyperparameters on clustering performance, we conduct a set of experiments with varying $s$ and $\gamma$ and illustrate the results in Figure~\ref{fig:sensitivity}.
For each hyperparameter setting, the experiments are repeated for five trials %
and report the average %
clustering accuracy.
As can be observed, the clustering performance of DeepMORSE remains stable across different settings of $s$ and $\gamma$, %
showing its robustness to hyperparameter choices.

\begin{table}[!t]
    \centering
    \caption{\textbf{Learning Expressive Matrix $\mC$ via SENet.}}
    \label{tab:SENet}
    \resizebox{\linewidth}{!}{
    \begin{tabular}{l|ccccc}
    \toprule
    \rowcolor{myGray} & \multicolumn{5}{c}{Clustering result (ACC: \%)}\\
    \rowcolor{myGray}  \multirow{-2}{*}{$\mC$} & CIFAR-10 & CIFAR-20 & Dogs-15 & DTD-47 & UCF-101 \\
    \cmidrule(lr){1-1}\cmidrule(lr){2-6}
    $\mathcal{T}_b(\boldsymbol{Q}_\text{mix}^{\top}\boldsymbol{K}_\text{mix})$ & 91.7  & 53.9    & 76.9  & 51.8  & \textbf{71.9} \\
    $\mathcal{P}(\mZ^\top_\text{mix}\mZ_\text{mix})$ & \textbf{92.3}  & \textbf{63.3}  & \textbf{87.9}  &  \textbf{54.7}   &  \textbf{71.9} \\
    \bottomrule
    \end{tabular}%
    }
\end{table}

\myparagraph{Learning $\mC$ via SENet}
To further investigate the effect of the ``signed'' Sinkhorn-Knopp projection in obtaining the modality-shared self-expressive matrix, we design a variant that replaces the Sinkhorn-based coefficients of DeepMORSE with the strategy of SENet~\cite{Zhang:CVPR21-SENet}.
Specifically, we retain the learnable transformations $\f$ and $\g$ to produce modality-specific representations $\vz_\text{img}$ and $\vz_\text{text}$, and introduce two separate MLPs---a query net and a key net---following SENet's architecture.
For each mini-batch, we first compute the modality-specific query embeddings of the batch samples and the key embeddings of all training data via these networks, and then average the embeddings of the two modalities to obtain modality-shared representations $\boldsymbol{Q}_\text{mix}$ and $\boldsymbol{K}_\text{mix}$.
The modality-shared self-expressive matrix is then obtained by applying the adaptive soft-thresholding operation on the inner product of the modality-shared representations, \ie, $\mathcal{T}_b(\boldsymbol{Q}_\text{mix}^{\top}\boldsymbol{K}_\text{mix})$.
Experimental results are reported in Table~\ref{tab:SENet}. As can be read that, the %
self-expressive matrix obtained via the ``signed'' Sinkhorn-Knopp projection consistently outperforms the %
SENet based variant method %
on all datasets except UCF-101. 

\begin{table}[!t]
    \centering
    \caption{\textbf{Variants of learnable transformations.}}
    \label{tab:f_g_necessity}
    \resizebox{\linewidth}{!}{
    \begin{tabular}{l|ccccc}
    \toprule
    \rowcolor{myGray} & \multicolumn{5}{c}{Clustering result (ACC: \%)}\\
    \rowcolor{myGray}  \multirow{-2}{*}{$\f,\g$} & CIFAR-10 & CIFAR-20 & Dogs-15 & DTD-47 & UCF-101 \\
    \cmidrule(lr){1-1}\cmidrule(lr){2-6}
    N/A & 78.2  & 53.0    & 61.7  & 45.8  & 60.3 \\
    Linear & \underline{92.1}  & 59.2  & 82.3  & \underline{54.4}  & 70.8 \\
    MLP (2 layers) & \textbf{92.3}  & \underline{63.3}  & \underline{87.9}  & \textbf{54.7}    & \underline{71.9} \\
    MLP (3 layers) & \textbf{92.3}  & \textbf{63.5}  & \textbf{88.4}  &  53.5   &  \textbf{73.5} \\
    \bottomrule
    \end{tabular}%
    }
\end{table}

\subsection{Other Evaluations}
\label{sec:Evaluations_on_Representations}

\myparagraph{Evaluation on Using other Learnable Transformations} %
To validate the necessity of involving a set of learnable nonlinear transformations $\f,\g$, we conducted an ablation study with different transformation variants: 
a) without using $\f,\g$, \ie, the self-expressive model is learned directly on the frozen CLIP features; 
b) using linear $\f,\g$, \ie, $\f,\g$ are replaced by learnable linear projections; and 
c) using nonlinear $\f,\g$, \ie, using nonlinear transformations with different layers.
For each variant, we tune the parameter $\gamma\in\{0.5, 1, 5, 10, 50, 100, 150, 300, 500\}$ and report the best average clustering accuracy after repeating experiments with 5 random seeds. Experimental results are reported in Table \ref{tab:f_g_necessity}. 
As can be read, %
the performance without using $\f,\g$ degenerates significantly %
compared to DeepMORSE. This confirms that the learned transformations are indeed necessary. 
Interestingly, the performance without using $\f,\g$ %
still outperforms the result of CLIP+$k$-means in Table~\ref{tab:tab1}, indicating that the raw features are better approximated by a union of subspaces rather than by centroids.
When linear transformations are used, the clustering performance improves substantially over the result of ``without $\f,\g$'', indicating that even a linear projection can approximately transform data onto a union of subspaces, thereby enhancing the clustering performance.
When nonlinear transformations are adopted, we obtain the best performance across all datasets. If we further increase the number of layers, %
marginal performance improvements can be observed. This suggests that %
using deep nonlinear transformations is able to effectively linearize and separate the subspaces in our DeepMORSE framework.

\begin{table}[tb]
    \centering
    \caption{\textbf{MLLM for textual counterpart generation.}}
    \label{tab:MLLM}
    \resizebox{\linewidth}{!}{
    \begin{tabular}{l|ccccc}
        \toprule
        \rowcolor{myGray} Text & \multicolumn{5}{c}{Clustering result (ACC: \%)}\\
        \rowcolor{myGray}  generation  & CIFAR-10 & CIFAR-20 & Dogs-15 & DTD-47 & UCF-101 \\
        \cmidrule(lr){1-1}\cmidrule(lr){2-6}
        MLLM as $\mathcal{T}$ & 77.5  & 50.3  & 82.3  & \textbf{56.6} & \textbf{74.0}\\
        MLLM as $\mathcal{D}$ & \underline{92.0} & \underline{61.0} & \underline{83.6} & \underline{55.5} & \underline{73.9}\\
         WordNet & \textbf{92.3}  & \textbf{63.3} & \textbf{87.9}  & 54.7  & 71.9 \\
        \bottomrule
    \end{tabular}}
\end{table}

\myparagraph{Generating Textual Counterparts Using MLLM}
Note that directly leveraging Multimodal Large Language Models (MLLMs) is an appealing way to enrich textual counterparts for image clustering. %
Here we take Qwen-VL-Plus as a representative of MLLMs and conduct following two sets of experiments: a) using MLLM as $\mathcal{T}$, and b) using MLLM as $\mathcal{D}$. 
In the experiments of using MLLM as $\mathcal{T}$, for each image, we prompt Qwen-VL-Plus to generate a descriptive caption, then wrap this caption into a CLIP-style prompt (\eg, ``a photo of [caption]'') and encode it with the pretrained CLIP text encoder to obtain textual counterparts. 
In the experiments of using %
MLLM as $\mathcal{D}$, %
we collect the captions for all images in a dataset %
as a dataset-specific, semantically compact dictionary, and %
then generate the textual counterparts from this dictionary and train DeepMORSE exactly as in the original WordNet-based setting of our DeepMORSE.
We report the %
average clustering accuracy over 5 random trials in Table~\ref{tab:MLLM}.
As can be read, %
using MLLM-generated captions yields performance comparable to or even better than the original WordNet-based setup, especially on the DTD-47 dataset, which contain texture images that are beyond the description of WordNet. However, for CIFAR-10 and CIFAR-20, %
the method based on MLLM-generated captions frequently misclassifies images because their resolution is very low ($32\times32$), making reliable caption generation difficult.

\begin{table*}[tbh]
    \centering
    \caption{\textbf{Comparing DeepMORSE to multi-view clustering baselines on multimodal and multi-view datasets.}}
    \label{tab:multi-view}
    \resizebox{\textwidth}{!}{
    \begin{tabular}{l|ccccccccc|ccccccccc}
        \toprule
        \rowcolor{myGray} & \multicolumn{9}{c|}{\textbf{Multimodal clustering datasets}} & \multicolumn{9}{c}{\textbf{Multi-view clustering datasets}}\\ 
        \rowcolor{myGray}  & \multicolumn{3}{c}{ImageNet-Dogs} & \multicolumn{3}{c}{DTD-47} & \multicolumn{3}{c|}{UCF101} & \multicolumn{3}{c}{Scene15} & \multicolumn{3}{c}{LandUse21} & \multicolumn{3}{c}{NUS-WIDE} \\
        \rowcolor{myGray}\multirow{-3}{*}{Method}  & ACC   & NMI   & ARI   & ACC   & NMI   & ARI   & ACC   & NMI   & ARI & ACC   & NMI   & ARI   & ACC   & NMI   & ARI   & ACC   & NMI   & ARI \\
        \cmidrule(lr){1-1}\cmidrule(lr){2-4}\cmidrule(lr){5-7}\cmidrule(lr){8-10}\cmidrule(lr){11-13}\cmidrule(lr){14-16}\cmidrule(lr){17-19}
        SURE~\cite{Yang:TPAMI22} & 77.9 & 77.0 & 66.0 & 49.6 & 60.7 & 32.8 & 65.9 & \underline{82.0} & 57.8 &  41.0 & 43.2 & 25.0& 25.1 & 28.3 & 10.9 & 57.4 & 44.8 & 38.3\\
        GCFAgg~\cite{Yan:CVPR23} & 76.9 & 78.5 & 62.5 & \underline{52.4} & \underline{63.9} & \underline{35.7} & \underline{68.9} & 80.4 & \underline{61.2} &  42.2 & 42.5 & 24.4 & 27.5 & 31.3 & 14.0 & 41.1 & 32.1 & 18.6\\
        DIVIDE~\cite{Lu:AAAI24} & 74.4 & 74.4& 60.0 & 48.7 & 59.9 & 31.2 & 68.1 & 81.5 & 58.3 &  \textbf{49.1} & \textbf{48.7} & \textbf{31.6} & \textbf{32.3} & \textbf{39.7} & \textbf{18.1} & 45.1 & 30.9 & 19.4 \\
        CANDY~\cite{Guo:NeurIPS24-CANDY} & 81.6  & 80.4  & 71.1  & 51.5  & 62.5  & 35.5  & 56.2  & 77.0  & 47.4  & 42.0  & 41.6  & 24.7  & 30.6  & 36.5  & 16.2  & \underline{62.1}  & 49.0  & 37.0  \\
        COPER~\cite{Eisenberg:ICLR25-coper} &  \underline{82.0} & \underline{80.8}  & \underline{72.2}  & 50.1  & 61.3  & 33.8  &  56.5 & 77.8  &  47.6 & 40.7  & 42.0  & 25.0  & 31.0  & 35.9  & 16.1  &  \underline{62.1} & \underline{49.3}  & \underline{37.7} \\
         Ours  & \textbf{87.9}  & \textbf{83.8}  & \textbf{78.3}  & \textbf{54.7}  & \textbf{64.2}  & \textbf{39.3}  & \textbf{71.9}  & \textbf{85.6}  & \textbf{66.7}   & \underline{43.2}  & \underline{46.2}  & \underline{28.0}  & \underline{31.4}  & \underline{37.9}  & \underline{17.4}  & \textbf{63.4}  & \textbf{50.6}  & \textbf{43.2}  \\
        \bottomrule
    \end{tabular}}
\end{table*}

\begin{table*}[tbh]
    \centering
    \caption{\textbf{Extending DeepMORSE beyond vision-language data.} We conduct experiments with Wav2CLIP as the pretrained tri-modal architecture.}
    \label{tab:trimodal}
    \resizebox{0.8\textwidth}{!}{
    \begin{tabular}{ccc|cccccccccc}
        \toprule
        \rowcolor{myGray} &  & & \multicolumn{2}{c}{CIFAR-10} & \multicolumn{2}{c}{CIFAR-20} & \multicolumn{2}{c}{ImageNet-Dogs} & \multicolumn{2}{c}{DTD-47} & \multicolumn{2}{c}{UCF-101} \\
        \rowcolor{myGray} \multirow{-2}{*}{Image} & \multirow{-2}{*}{Text} & \multirow{-2}{*}{Audio}  & ACC   & NMI   & ACC   & NMI   & ACC   & NMI   & ACC   & NMI   & ACC   & NMI \\
        \cmidrule(lr){1-3}\cmidrule(lr){4-13}
        \ding{51} &  &  & 87.3 & 79.6 & 53.1  & 60.9 & 29.3  & 23.9 & 49.0  & 57.0 & 67.9  & 83.1 \\
        &  \ding{51} &  & 77.6 & 60.9 & 34.8  & 34.5 &  69.2 & 69.9 & 39.6  & 51.9 & 60.3 & 76.0 \\
        &   & \ding{51} & 77.4 & 62.5 & 32.0  & 30.6 &  19.9 & 15.5 & 24.1  & 34.6 & 42.1  & 63.3 \\
        \ding{51}& \ding{51} & & \underline{92.3}  & \underline{84.1}  & \underline{63.3}  & \textbf{63.1}  & \underline{87.9}  & \underline{83.8}  & \underline{54.7}  & \underline{64.2}  & \underline{71.9}  & \underline{85.6} \\
        \ding{51} & &\ding{51} & 89.0  & 80.5  & 57.1  & \underline{60.6}  & 30.6  & 27.0  & 44.6  & 54.1  & 67.4  & 83.3  \\
         \ding{51}& \ding{51}& \ding{51} & \textbf{92.6}  & \textbf{84.4}  & \textbf{63.7}  & \textbf{63.1}  & \textbf{88.7}  & \textbf{84.5}  & \textbf{54.9}  & \textbf{64.7}  & \textbf{72.6}  & \textbf{85.9}  \\
        \bottomrule
    \end{tabular}}
\end{table*}

\begin{table*}[htbp]
\centering
\caption{\textbf{Image retrieval performance (mAP: \%) comparing to the state-of-the-art.} The comparative results are directly cited. The image encoder of all methods are based on ViT-B/32.}
\label{tab:image-retrieval}
\resizebox{0.9\linewidth}{!}{%
\begin{tabular}{l|ccccccccccc}
\toprule
\rowcolor{myGray}
Method & Flower & DTD & Pets & Cars & UCF & Caltech & Food & SUN & Aircraft & EuroSAT & ImageNet-1k \\ 
\cmidrule(lr){1-1}\cmidrule(lr){2-12}
CLIP ViT-B/32~\cite{Radford:ICML21-CLIP} & 62.0 & 28.1 & 30.5 & 24.6 & 47.1 & 77.1 & 32.3 & 34.3 & \underline{14.5} & \underline{47.9} & 21.4\\
OTI~\cite{Mistretta:ICLR25} &  \underline{62.6} & \underline{31.9} & \underline{37.5} & \underline{28.0} & \underline{48.6} & \underline{79.9} & \underline{34.7} & \underline{36.3} & 14.4 & 47.2 & \underline{23.8}\\
Our DeepMORSE & \textbf{65.2} & \textbf{32.0} & \textbf{57.7} & \textbf{28.9} & \textbf{53.3} & \textbf{81.1} & \textbf{57.1} & \textbf{42.6} & \textbf{15.6} & \textbf{53.4} & \textbf{30.2}\\
\bottomrule
\end{tabular}%
}
\end{table*}

\begin{table*}[tbp]
\centering
\caption{\textbf{Zero-shot classification performance (ACC: \%) comparing to the state-of-the-art.} The comparative results are directly cited. The image encoder of all methods are based on ViT-B/16.}
\label{tab:zero-shot}
\resizebox{0.9\linewidth}{!}{%
\begin{tabular}{l|ccccccccccc}
\toprule
\rowcolor{myGray}
Method & Flower & DTD & Pets & Cars & UCF & Caltech & Food & SUN & Aircraft & EuroSAT & ImageNet-1k \\ 
\cmidrule(lr){1-1}\cmidrule(lr){2-12}
CLIP ViT-B/16~\cite{Radford:ICML21-CLIP} & 64.4 & 44.3 & 88.3 & 65.5 & 65.1 & 93.4 & 83.7 & 62.6 & 23.7 & 42.0 & 66.7\\
Ensemble \cite{Zhang:ECCV22} & 67.0 & 45.0 & 86.9 & 66.1 & 65.2 & 93.6 & 82.9 & 65.6 & 23.2 & 50.4 & 68.3\\
TPT \cite{Shu:NIPS22} & 69.0 & 47.8 & 87.8 & 66.9 & 68.0 & 94.2 & 84.7 & 65.5 & 24.8 & 42.4 & 69.0\\
DiffTPT \cite{Feng:ICCV23} & 70.1&47.0&88.2&67.0&68.2&92.5&\textbf{87.2}&65.7&25.6&43.1 & 70.3\\
DMN \cite{Zhang:CVPR24} & 75.3 & \underline{54.9} & 91.2 & 67.0 & 72.0 & 93.6 & 84.1 & 69.1 & \underline{28.3} & 56.2 & 70.5\\ 
TDA \cite{Karmanov:CVPR24} & 71.4 & 47.4 & 88.6 & 67.3 & 70.7 & \underline{94.2} & 86.1 & 67.6 & 23.9 & \textbf{58.0} & 69.5\\
ZLaP \cite{Kalantidis:CVPR24} & 73.5 & 48.6 & 87.1 & 65.6 & 71.5 & 93.1 & \underline{86.9} & 67.4 & 25.4 & 55.6&70.2\\
ECALP \cite{Li:ICLR25} & \underline{76.0} & \textbf{56.3} & \underline{92.3} & \underline{68.2} & \textbf{75.4} & \textbf{94.4} & 85.7 & \textbf{70.5} & \textbf{29.5} & \underline{56.5} & \underline{71.3}\\
Our DeepMORSE & \textbf{76.9} & 54.8 & \textbf{93.0} & \textbf{68.7} & \underline{74.1} & 91.9 & 86.2 & \underline{70.2} & 27.1 & 44.4 & \textbf{71.4}\\
\bottomrule
\end{tabular}%
}
\end{table*}

\myparagraph{Compared to Multi-view Clustering}
To thoroughly evaluate the performance of our DeepMORSE in the context of multi-view clustering, we conduct the following two sets of experiments. 
First, we treat the CLIP image and text embeddings as two views, then apply state-of-the-art multi-view clustering methods, SURE~\cite{Yang:TPAMI22}, GCFAgg~\cite{Yan:CVPR23}, DIVIDE~\cite{Lu:AAAI24}, CANDY~\cite{Guo:NeurIPS24-CANDY}, and COPER~\cite{Eisenberg:ICLR25-coper}, to our multimodal data. 
For a fair comparison, we keep network architectures the same and repeat all experiments with 5 random seeds, reporting the average %
performance in Table~\ref{tab:multi-view}. 
Second, we evaluate our DeepMORSE directly on several multi-view clustering datasets, including 
Scene15~\cite{Fei:CVPR05}, LandUse21~\cite{Yang:AGIS10}, NUS-WIDE~\cite{Hu:RDIR19}, where we treat the multiple views as separate modalities within our framework. The experimental results are reported in Table~\ref{tab:multi-view}. 

Across all datasets, we observe that our DeepMORSE is highly competitive with the multi-view baselines, implying that modeling modality-specific UoS structures with modality-shared relationship, as done in our DeepMORSE, provides a more effective way to leverage both multimodal and multi-view information.

\myparagraph{Extending Beyond Vision-Language Data}
While our current experiments focus on vision-language data, the DeepMORSE framework itself is modality-agnostic.
Specifically, our DeepMORSE only requires that each modality be mapped into a shared semantic feature space (as in CLIP). Given multimodal input features, we can learn a transformation for each modality (\eg, image, text, acoustics, hyperspectral) via separate encoders, and construct a modality-shared self-expressive matrix without any architectural changes. In this sense, our DeepMORSE can be straightforwardly extended to additional modalities once suitable multimodal features are available.
To demonstrate this, we additionally conduct experiments with Wav2CLIP~\cite{Wu:ICASSP22-Wav2CLIP}, which provides a joint image-text-audio embedding space. Concretely, we first generate audio counterparts for each image via multimodal sparse coding (as in Sec. \ref{Sec:Implementations}) from the audio features of the FSD50K dataset~\cite{Fonseca:TASLP22-FSD50K}, which contains 51,197 Freesound clips across 200 classes. 
For training our DeepMORSE, we transform the audio data $\boldsymbol{u}$ into representations with another learnable mapping $\boldsymbol{z}_\text{audio}=\boldsymbol{h}(\boldsymbol{u};\boldsymbol{\phi}_{\boldsymbol{u}})$ followed by normalization.
We then construct a modality-mixed representation:  $\vz_\text{mix}=0.5(\vz_\text{img}+\vz_\text{audio})$ for image-audio; and $\vz_\text{mix}=0.45\times\vz_\text{img}+0.45\times\vz_\text{text} + 0.1\times\vz_\text{audio}$ for image-text-audio.
The modality-shared self-expressive matrix is still computed by Eq. (\ref{eq:compute_C}), and we add the audio loss term $\gamma\|\boldsymbol{Z}_\text{audio}-\boldsymbol{Z}_\text{audio}\boldsymbol{C}\|_F^2-\rho(\boldsymbol{Z}_\text{audio})$ into the overall objective.
The results are shown in Table~\ref{tab:trimodal}. 
We observe a notable performance drop for the Image$+$audio combination, which is consistent with the findings in AudioCLIP~\cite{Guzhov:ICASSP22-AudioCLIP}, where ImageNet top-1 accuracy drops from 40.5\% to 21.8\%.
This degradation is largely due to the fact that existing tri-modal encoders are trained on much smaller audio corpora (\eg, AudioSet with ~2M clips) compared to the large-scale image-text data used for CLIP (400M pairs), resulting in weaker image representations.

For the Image+text+audio combination, the performance on CIFAR-10, CIFAR-20, and ImageNet-Dogs slightly improves or remains on par with the Image$+$text setting, demonstrating that DeepMORSE can incorporate an additional audio modality without harming performance and sometimes with small gains.
For DTD-47 and UCF-101, audio is unsurprisingly less helpful: textures (\eg, striped, grid) and many actions (\eg, applying eye makeup, playing yo-yo) are essentially silent.

\myparagraph{DeepMORSE for Image Retrieval}
Image retrieval is a unimodal downstream task that aims to retrieve images belonging to the same category. 
Here, we directly perform the retrieval based on the cosine similarity of the image representations learned by our DeepMORSE.
As shown in Table~\ref{tab:image-retrieval}, our DeepMORSE achieves state-of-the-art performance on image retrieval without any task-specific design. 
Prior work~\cite{Mistretta:ICLR25} has shown that intra-modal similarity scores from CLIP models do not always correspond faithfully to semantic similarities among images. 
In contrast, the representations learned by our DeepMORSE mitigate this issue by aligning with structures shared across modalities.

\myparagraph{DeepMORSE for Zero-Shot Classification}
Zero-shot classification is a multimodal downstream task that assigns each image to a class based on the similarity between its representation and the text representations of the given labels.
Since that the categories of the datasets are known, we construct the dictionary $\mathcal{D}=\{\boldsymbol{d}_1,\cdots,\boldsymbol{d}_{C}\}$ for cross-modal sparse coding using embeddings of the prompts ``A photo of [class name]'' extracted by a pretrained CLIP text encoder, where $C$ is the number of classes.
After training DeepMORSE, each image $\vx$ is classified according to the similarity between its learned representation and the class representations, \ie, 
    $\mathop{\arg\max}_{j\in\{1,\cdots,C\}}\quad \frac{f(\vx)^\top g(\boldsymbol{\pi}_j)}{\|f(\vx)\|_2\|g(\boldsymbol{\pi}_j)\|_2}$.
The results are reported in Table~\ref{tab:zero-shot}, where the baseline results are directly cited from~\cite{Li:ICLR25}.
As can be read, %
our DeepMORSE achieves competitive performance and surpasses existing methods on datasets Flowers102, OxfordPets, and StanfordCars---without any task-specific design for zero-shot classification.

\section{Conclusion}
\label{Sec:Conclusion}

We have presented a simple but principled approach for deep image clustering assisted with textual information, called DeepMORSE, which jointly learns representations that conform to a union of modality-specific subspaces and discovers shared self-expressive coefficients across modalities.
Moreover, %
we have theoretically justified %
that the modality-shared self-expressive coefficients under certain conditions suppress the inter-class noise compared to modality-specific solutions, leading to a provably cleaner affinity for clustering. 
In addition, we have demonstrated that the SGD training in mini-batch mode intrinsically introduces an implicit %
regularization on the coefficient matrix, %
with no need to incorporate an explicit regularization term. %
Empirically, we have shown that DeepMORSE achieves state-of-the-art clustering performance on six widely used benchmarks, and the structured representations learned by DeepMORSE transfer effectively to downstream tasks including image retrieval and zero-shot classification, achieving competitive results without any task-specific losses or post-processing.
In future work, it would be interesting to extend DeepMORSE to scenarios with naturally available text, to larger-scale problems, and to downstream tasks such as few-shot learning and open-world recognition.

\section*{Acknowledgments}
This work is supported by the National Natural Science Foundation of China under Grant 62576048.

\bibliographystyle{IEEEtran}
\bibliography{./biblio/xhmeng,./biblio/temporal,./biblio/yc,./biblio/cgli,./biblio/zhjj,./biblio/learning,./biblio/temp}

\onecolumn
\setcounter{page}{1}
\setcounter{section}{0}

\renewcommand{\thesection}{\Alph{section}}

\setcounter{figure}{0}
\setcounter{table}{0}
\setcounter{proposition}{0}

\renewcommand{\thefigure}{\Alph{section}.\arabic{figure}}
\renewcommand{\thetable}{\Alph{section}.\arabic{table}}
\begin{sc}
\begin{center}
    \Large Supplementary Material for ``Learning Deep Modality-Shared Self-Expressiveness for Image Clustering with Textual Information''
\end{center}
\end{sc}

\section{Proofs of Theoretical Findings}
\label{sec:supp_theory}

\subsection{Proof for Proportion 1}

\begin{proof}

Since that problem (\ref{eq:subprob}) is convex with respect to $\vc^j$, we get its closed-form solution by the first-order condition:
\begin{align}
    \nabla_{\vc^j}\mathcal{L} &= \Big((\mZ_\text{img}^{-j})^\top \mZ_\text{img}^{-j} + (\mZ_\text{text}^{-j})^\top \mZ_\text{text}^{-j}\Big)\vc^j-(\mZ_\text{img}^{-j})^\top\vz_\text{img}^j
             -(\mZ_\text{text}^{-j})^\top\vz_\text{text}^j=\boldsymbol{0},\\
    \Rightarrow\vc_\text{share}^j&= (\mG_\text{img}+\mG_\text{text})^{-1}
      \left[(\mZ_\text{img}^{-j})^\top\vz_\text{img}^j
             +(\mZ_\text{text}^{-j})^\top\vz_\text{text}^j\right],     
\end{align}
where $\mG_\text{img} \coloneqq (\mZ_\text{img}^{-j})^\top \mZ_\text{img}^{-j}$
and $\mG_\text{text} \coloneqq (\mZ_\text{text}^{-j})^\top \mZ_\text{text}^{-j}$ are Gram matrices.
Similarly, the modality-specific optimal solutions of problem (\ref{eq:subprob}) are:
\begin{equation}
    \vc_\text{img}^j \coloneqq \mG_\text{img}^{-1}(\mZ_\text{img}^{-j})^\top\vz_\text{img}^j
    \quad \text{and}\quad
    \vc_\text{text}^j \coloneqq \mG_\text{text}^{-1}(\mZ_\text{text}^{-j})^\top\vz_\text{text}^j,
\end{equation}
respectively.

\myparagraph{Proof of Statement 1}\quad
Leveraging eq. (\ref{eq:c_res}), the estimation errors are
\begin{align}
    \vc_\text{img}^j   - \vc_\star^j
        &= \mG_\text{img}^{-1}(\mZ_\text{img}^{-j})^\top(\mZ_\text{img}^{-j}\vc_\star^j+\vdelta_\text{img}^j)-\vc_\star^j=\mG_\text{img}^{-1}(\mZ_\text{img}^{-j})^\top\vdelta_\text{img}^j,   \\
    \vc_\text{text}^j  - \vc_\star^j
        &=\mG_\text{text}^{-1}(\mZ_\text{text}^{-j})^\top(\mZ_\text{text}^{-j}\vc_\star^j+\vdelta_\text{text}^j)-\vc_\star^j= \mG_\text{text}^{-1}(\mZ_\text{text}^{-j})^\top\vdelta_\text{text}^j, \\
    \vc_\text{share}^j - \vc_\star^j
        &= (\mG_\text{img}+\mG_\text{text})^{-1}
           \!\left[(\mZ_\text{img}^{-j})^\top\vdelta_\text{img}^j
                  +(\mZ_\text{text}^{-j})^\top\vdelta_\text{text}^j\right].
\end{align}
Then, we have:
\begin{align}
    &\quad\quad\mathbb{E}\left[\|\vc_\text{share}^j-\vc_\star^j\|_2^2\right]\\
    &= \mathbb{E}\left[\Tr\left(\left(\vc_\text{share}^j - \vc_\star^j\right)\left(\vc_\text{share}^j - \vc_\star^j\right)^\top\right)\right]\\
    &= \Tr\left(\left(\mG_\text{img}+\mG_\text{text}\right)^{-1}\left(\left(\mZ_\text{img}^{-j}\right)^\top\mathbb{E}\left[\vdelta_\text{img}^j\vdelta_\text{img}^{j\top}\right]\mZ_\text{img}^{-j}+\left(\mZ_\text{text}^{-j}\right)^\top\mathbb{E}\left[\vdelta_\text{text}^j\vdelta_\text{text}^{j\top}\right]\mZ_\text{text}^{-j}\right)\left(\mG_\text{img}+\mG_\text{text}\right)^{-1}\right)\\
    &=\Tr\left(\left(\mG_\text{img}+\mG_\text{text}\right)^{-1}\left(\sigma^2\mG_\text{img}+\sigma^2\mG_\text{text}\right)\left(\mG_\text{img}+\mG_\text{text}\right)^{-1}\right)\\
    &=\sigma^2\Tr\left((\mG_\text{img}+\mG_\text{text})^{-1}\right).
\end{align}
Similarly, we have
\begin{equation}
    \mathbb{E}\!\left[\|\vc_\text{img}^j-\vc_\star^j\|_2^2\right]
    = \sigma^2\operatorname{Tr}(\mG_\text{img}^{-1}),
    \quad \text{and} \quad
    \mathbb{E}\!\left[\|\vc_\text{text}^j-\vc_\star^j\|_2^2\right]
    = \sigma^2\operatorname{Tr}(\mG_\text{text}^{-1}).
\end{equation}
Since $\mG_\text{text}\succeq\mathbf{0}$, we have
$\mG_\text{img}\preceq\mG_\text{img}+\mG_\text{text}$, therefore,
\begin{equation}
    (\mG_\text{img}+\mG_\text{text})^{-1}\preceq\mG_\text{img}^{-1}.
\end{equation}
Taking the trace of both sides and
multiplying by $\sigma^2$ yields
\begin{equation}
    \mathbb{E}\left[\left\|\vc_\text{share}^j-\vc_\star^j\right\|_2^2\right] \leq\mathbb{E}\left[\left\|\vc_\text{img}^j-\vc_\star^j\right\|_2^2\right].
\end{equation}
Similarly, %
we have that $\mathbb{E}\left[\left\|\vc_\text{share}^j-\vc_\star^j\right\|_2^2\right] \leq\mathbb{E}\left[\left\|\vc_\text{img}^j-\vc_\star^j\right\|_2^2\right]$. 
Thus we conclude the proof
\begin{equation}
    \mathbb{E}\Big[\|\vc_\text{share}^j-\vc_\star^j\|_2^2\Big]\leq\mathop{\min}\Big(\mathbb{E}\Big[\|\vc_\text{img}^j-\vc_\star^j\|_2^2\Big],\mathbb{E}\Big[\|\vc_\text{text}^j-\vc_\star^j\|_2^2\Big]\Big).
\end{equation}

\myparagraph{Proof of Statement 2}\quad
Let $\mathcal{I}_+^j, \mathcal{I}_-^j\subset\mathbb{N}_{\geq0}$ be the index sets of samples belonging to the same class as the $j$-th data point and belonging to a different class from the $j$-th data point, respectively. 
Denote $\Diag(\mathbbm{1}_{\mathcal{I}_+^j})$ as the diagonal matrix whose $i$-th diagonal entry is $1$ if $i\in\mathcal{I}_+^j$ and $0$ otherwise.

Since $\vc_\star^j$ is subspace-preserving, we have that 
$\vc_{\star,\mathcal{I}_-^j}^j\coloneqq(\mathbf{I}_N-\Diag(\mathbbm{1}_{\mathcal{I}_+^j}))\vc_{\star}^j=\Diag(\mathbbm{1}_{\mathcal{I}_-^j})\vc_{\star}^j=\mathbf{0}$. 
Thus, 
\begin{align}
\begin{aligned}
\mathbb{E}\left[\left\|\vc_{\text{share},\mathcal{I}_-^j}^j\right\|_2^2\right]&=\mathbb{E}\left[\left\|\Diag(\mathbbm{1}_{\mathcal{I}_-^j})\vc_{\text{share}}^j\right\|_2^2\right]\\
&=\mathbb{E}\left[\left\|\Diag(\mathbbm{1}_{\mathcal{I}_-^j})\vc_{\text{share}}^j-\Diag(\mathbbm{1}_{\mathcal{I}_-^j})\vc_{\star}^j\right\|_2^2\right]\\
&=\mathbb{E}\left[\left\|\Diag(\mathbbm{1}_{\mathcal{I}_-^j})(\vc_{\text{share}}^j-\vc_{\star}^j)\right\|_2^2\right]\\
&=\Tr\left(\Diag\left(\mathbbm{1}_{\mathcal{I}_-^j}\right)\mathbb{E}\left[\left(\vc_{\text{share}}^j-\vc_{\star}^j\right)\left(\vc_{\text{share}}^j-\vc_{\star}^j\right)^\top\right]\Diag\left(\mathbbm{1}_{\mathcal{I}_-^j}\right)\right)\\
&=\sigma^2\Tr\left(\Diag\left(\mathbbm{1}_{\mathcal{I}_-^j}\right)\left(\mG_\text{img}+\mG_\text{text}\right)^{-1}\Diag\left(\mathbbm{1}_{\mathcal{I}_-^j}\right)\right)\\
&=\sigma^2\Tr\left(\left(\mG_\text{img}+\mG_\text{text}\right)^{-1}_{\mathcal{I}_-^j,\mathcal{I}_-^j}\right),
\end{aligned}
\end{align}
and similarly we have: 
\begin{equation}
    \mathbb{E}\left[\left\|\vc_{\text{img},\mathcal{I}_-^j}^j\right\|_2^2\right]=
    \sigma^2\Tr\left((\mG_{\text{img}}^{-1})_{\mathcal{I}_-^j,\mathcal{I}_-^j}\right),\quad\quad\mathbb{E}\left[\left\|\vc_{\text{text},\mathcal{I}_-^j}^j\right\|_2^2\right]=\sigma^2\Tr\left((\mG_{\text{text}}^{-1})_{\mathcal{I}_-^j,\mathcal{I}_-^j}\right),
\end{equation}
where $[\cdot]^{-1}_{\mathcal{I}_-^j,\mathcal{I}_-^j}$ denotes the submatrix of $[\cdot]^{-1}$ where rows and columns are indexed by $\mathcal{I}_-^j$. 

Note that if $\mA\preceq\mB$, we have $\mA_{\mathcal{S},\mathcal{S}}\preceq\mB_{\mathcal{S},\mathcal{S}}$, since that for $\forall \vv \in \mathbb{R}^{|\mathcal{S}|}$, %
we have
\begin{equation}
    \vv^\top\mA_{\mathcal{S},\mathcal{S}}\vv = \tilde{\vv}^\top\mA\tilde{\vv}\leq\tilde{\vv}^\top\mB\tilde{\vv} = \vv^\top\mB_{\mathcal{S},\mathcal{S}}\vv,
\end{equation}
where $\tilde{\vv}\in\mathbb{R}^N$ is constructed by zero-padding: $\tilde{\vv}_i = \vv_{i}$ if $i\in\mathcal{S}$ and $\tilde{\vv}_i = 0$ otherwise.

Therefore, we conclude that
\begin{equation}
\mathbb{E}\Big[\|\vc_{\text{share},\mathcal{I}_-^j}^j\|_2^2\Big]\leq\mathop{\min}\Big(\mathbb{E}\Big[\|\vc_{\text{img},\mathcal{I}_-^j}^j\|_2^2\Big],\mathbb{E}\Big[\|\vc_{\text{text},\mathcal{I}_-^j}^j\|_2^2\Big]\Big),
\end{equation}
which completes the proof.
\end{proof}

\setcounter{proposition}{2}

\subsection{Proof for Proportion 3}

\begin{proof}
We have:
\begin{align}
\begin{aligned}
    \tilde{\mathcal{L}} \coloneqq & \|\mZ_{\text{img}}\Diag(\boldsymbol{\xi}) - \mZ_{\text{img}}\Diag(\boldsymbol{\xi})\mC\Diag(\boldsymbol{\xi})\|_F^2 + \|\mZ_{\text{text}}\Diag(\boldsymbol{\xi}) - \mZ_{\text{text}}\Diag(\boldsymbol{\xi})\mC\Diag(\boldsymbol{\xi})\|_F^2\\
    &=\sum_{j=1}^N\|\xi_j\vz_\text{img}^j-\xi_j[\mZ_\text{img}\Diag(\boldsymbol{\xi})\mC]_j\|_2^2+\|\xi_j\vz_\text{text}^j-\xi_j[\mZ_\text{text}\Diag(\boldsymbol{\xi})\mC]_j\|_2^2\\
    &=\sum_{j=1}^N\xi_j\Big(\|\vz_\text{img}^j-\sum_{i=1}^N\xi_ic_{ij}\vz_\text{img}^i\|_2^2+\|\vz_\text{text}^j-\sum_{i=1}^N\xi_ic_{ij}\vz_\text{text}^i\|_2^2\Big).
\end{aligned}
\end{align}
By taking an expectation to the objective function, we have: 
\begin{align}
\begin{aligned}
    \mathbb{E}[\tilde{\mathcal{L}}]=\sum_{j=1}^N\mathbb{E}[\xi_j]\Big(\mathbb{E}\Big[\|\vz_\text{img}^j-\sum_{i=1}^N\xi_ic_{ij}\vz_\text{img}^i\|_2^2\Big]+\mathbb{E}\Big[\|\vz_\text{text}^j-\sum_{i=1}^N\xi_ic_{ij}\vz_\text{text}^i\|_2^2\Big]\Big).
\end{aligned}
\end{align}
The first and second moments of the random variable $\xi$ are given by $\mathbb{E}[\xi_i]=\frac{n_b}{N}$, $\mathbb{E}[\xi_i\xi_j]=\mathbb{E}[\xi_i]\mathbb{E}[\xi_j]=\frac{n_b^2}{N^2}$, $\mathbb{E}[{\xi_i^2}]=\frac{n_b}{N}$.
Then, we have:
\begin{align}
\begin{aligned}
    \mathbb{E}\left[\left\|\vz_\text{img}^j-\sum_i\xi_ic_{ij}\vz_{\text{img}}^i\right\|_2^2\right]&=\|\vz_\text{img}^j\|_2^2-2\mathbb{E}\sum_i\xi_ic_{ij}\vz_\text{img}^{j\top}\vz_\text{img}^i+\mathbb{E}\left[\left\|\sum_i\xi_ic_{ij}\vz_\text{img}^i\right\|_2^2\right]\\
    &=\|\vz_\text{img}^j\|_2^2-2\sum_i\mathbb{E}\left[\xi_i\right]c_{ij}\vz_\text{img}^{j\top}\vz_\text{img}^i+\sum_i\mathbb{E}\|\xi_ic_{ij}\vz_\text{img}^i\|_2^2+2\mathbb{E}\sum_{i,k:i\neq k}\xi_i\xi_k c_{ij}c_{kj}\vz_\text{img}^{i\top}\vz_\text{img}^k\\
    &=\|\vz_\text{img}^j\|_2^2-2\sum_i\frac{n_b}{N}c_{ij}\vz_\text{img}^{j\top}\vz_\text{img}^i+\sum_i\mathbb{E}[\xi_i^2]\|c_{ij}\vz_\text{img}^i\|_2^2+2\sum_{i,k:i\neq k}\mathbb{E}[\xi_i\xi_k] c_{ij}c_{kj}\vz_\text{img}^{i\top}\vz_\text{img}^k\\
    &=\|\vz_\text{img}^j\|_2^2-2\sum_i\frac{n_b}{N}c_{ij}\vz_\text{img}^{j\top}\vz_\text{img}^i+\sum_i\frac{n_b}{N}\|c_{ij}\vz_\text{img}^i\|_2^2+2\sum_{i,k:i\neq k}\frac{n_b^2}{N^2} c_{ij}c_{kj}\vz_\text{img}^{i\top}\vz_\text{img}^k\\
    &=\|\vz_\text{img}^j\|_2^2-2\sum_i\frac{n_b}{N}c_{ij}\vz_\text{img}^{j\top}\vz_\text{img}^i+\left\|\sum_i\frac{n_b}{N}c_{ij}\vz_\text{img}^i\right\|_2^2+(\frac{n_b}{N}-\frac{n_b^2}{N^2})\sum_i\|c_{ij}\vz_\text{img}^i\|_2^2\\
    &=\left\|\vz_\text{img}^j-\sum_i\frac{n_b}{N}c_{ij}\vz_\text{img}^i\right\|_2^2+(\frac{n_b}{N}-\frac{n_b^2}{N^2})\|\vc_j\|_2^2.
\end{aligned}
\end{align}
Let $\tilde\vc_j=\frac{n_b}{N}\vc_j$, we have:
\begin{align}
\begin{aligned}
    \mathbb{E}\left[\left\|\vz_\text{img}^j-\sum_i\xi_ic_{ij}\vz_{\text{img}}^i\right\|_2^2\right]&=\left\|\vz_\text{img}^j-\mZ_\text{img}\tilde\vc_j\right\|_2^2+\frac{N-n_b}{n_b}\|\tilde\vc_j\|_2^2.
\end{aligned}
\end{align}
Thus, the expectation of the objective function is equal to:
\begin{align}
\begin{aligned}
    \mathbb{E}[\tilde{\mathcal{L}}]&=\frac{n_b}{N}\sum_{j=1}^N\Big(\left\|\vz_\text{img}^j-\mZ_\text{img}\tilde\vc_j\right\|_2^2+\left\|\vz_\text{text}^j-\mZ_\text{text}\tilde\vc_j\right\|_2^2+\frac{2(N-n_b)}{n_b}\|\tilde\vc_j\|_2^2\Big)\\
    &=\frac{n_b}{N}\Big(\left\|\mZ_\text{img}-\mZ_\text{img}\tilde\mC\right\|_F^2+\left\|\mZ_\text{text}-\mZ_\text{text}\tilde\mC\right\|_F^2+\frac{2(N-n_b)}{n_b}\|\tilde\mC\|_F^2\Big).
\end{aligned}
\end{align}
Therefore, we conclude that optimizing the loss function (\ref{eq:loss_func}) with mini-batch training asymptotically introduces an implicit regularization $r(\mC)=\frac{2(N-n_b)}{n_b}\|\mC\|_F^2$.

\end{proof}

\section{Experimental Details}
\label{sec:Experimental_Details}

\myparagraph{Construction of Dictionary $\mathcal{D}$}
Since WordNet contains 146,347 words with heterogeneous semantic meanings, we follow TAC~\cite{Li:ICML24-TAC} to filter the vocabulary before textual counterpart generation.
Specifically, after extracting image and text embeddings using a pretrained vision-language model, we apply $k$-means clustering to the image embeddings to obtain $k$ cluster centers.
For each cluster center, the top-$\omega$ most similar word embeddings are then selected to construct the dictionary for cross-modal sparse coding.
Following TAC~\cite{Li:ICML24-TAC}, we fix $k=N/300$ and $\omega=5$ for all datasets.

\myparagraph{Image Clustering}
For a fair comparison, we follow the preprocessing steps of TAC~\cite{Li:ICML24-TAC}.
Specifically, images are first resized to 224 pixels on the shorter side (256 for ImageNet-10 and ImageNet-Dogs), then center-cropped to $224\times224$, and finally fed into the CLIP ViT-B/32 pretrained image encoder to obtain image embeddings.
The construction of text embedding dictionary $\mathcal{D}$ also follows TAC~\cite{Li:ICML24-TAC}.
With the query image embeddings and the text embedding dictionary $\mathcal{D}$, textual counterparts are generated via cross-modal sparse coding.

We summarize the hyperparameters used to train DeepMORSE in Table~\ref{table:hyperparameters}.
The model is trained with stochastic gradient descent (SGD) using a learning rate of $\eta=1\times10^{-4}$, weight decay $wd=1\times10^{-4}$, and batch size $n_b=1024$. The hidden dimension $d_\text{hidden}=512$ and output dimension $d=128$ of $\f(~\cdot~;\phi_I)$ and $\g(~\cdot~;\phi_T)$ are fixed across all datasets. Similarly, the sparsity level of sparse coding is set to $s=5$, and the loss hyperparameters are fixed to $\gamma=150$ and $\epsilon^2=0.1$ for all datasets.

\begin{table*}[h]
\caption{\textbf{Hyperparameters configuration for training DeepMORSE for image clustering.}}
\label{table:hyperparameters}
\begin{center}
\resizebox{0.7\linewidth}{!}{
    \begin{tabular}{lccccccccc}
    \toprule
         \rowcolor{myGray} & $\eta$ & $wd$ & $d_\text{hidden}$&  $d$&  $T$& $n_b$ &$\gamma$ & $\epsilon^2$ & $s$\\
         \midrule
         CIFAR-10& $1\times 10^{-4}$ & $1\times 10^{-4}$ & 512&  128&  50& 1024& 150 & 0.1 & 5\\
         CIFAR-20& $1\times 10^{-4}$ & $1\times 10^{-4}$& 512&  128&  50& 1024& 150& 0.1 & 5\\
         STL-10& $1\times 10^{-4}$ & $1\times 10^{-4}$& 512&  128&  50& 1024& 150& 0.1 & 5\\
         ImageNet-10& $1\times 10^{-4}$& $1\times 10^{-4}$ & 512&  128&  50& 1024& 150& 0.1 & 5\\
         ImageNet-Dogs& $1\times 10^{-4}$ & $1\times 10^{-4}$& 512&  128&  50& 1024& 150& 0.1 & 5\\
         DTD-47& $1\times 10^{-4}$ & $1\times 10^{-4}$& 512&  128&  100& 1024& 150& 0.1 & 5\\
         UCF-101& $1\times 10^{-4}$ & $1\times 10^{-4}$& 512&  128&  100& 1024& 150& 0.1 & 5\\
         ImageNet-1k& $1\times 10^{-4}$ & $1\times 10^{-4}$& 8192&  1024&  500& 8192& 600& 0.1 & 5\\
          \bottomrule
    \end{tabular}
    }
\end{center}

\end{table*}

\myparagraph{Image Retrieval}
For a fair comparison, we directly use the image embeddings provided by OTI~\cite{Mistretta:ICLR25}, which were extracted using the pretrained CLIP ViT-B/32 model.
To generate the textual counterpart of each image embedding, we first construct the text embedding dictionary as described in Section~\ref{Sec:Experiments}, and then solve the cross-modal sparse coding problem with MP.
After training DeepMORSE, image retrieval is performed based on the cosine similarity among the learned image representations.

We summarize the hyperparameters for training DeepMORSE on the image retrieval task in Table~\ref{table:hyperparameters-retrieval-zs}, where most hyperparameters remain unchanged, except for extended epochs and larger $d,\gamma$ for datasets with greater than 128 classes.
Specifically, since the benchmark datasets are relatively small in size (for example, OxfordPets, Flowers102, FGVCAircraft, DTD, and Caltech101 each contain fewer than 5,000 training samples), we extend the training epochs to ensure sufficient iterations.
For StanfordCars and SUN397, which include more than 128 categories, we increase the hidden and output dimensions of the model and adjust $\gamma$ accordingly to balance the different loss terms.

\begin{table*}[h]
\caption{\textbf{Hyperparameters configuration for training DeepMORSE for image retrieval and zero-shot classification.}}
\label{table:hyperparameters-retrieval-zs}
\begin{center}
\resizebox{0.7\linewidth}{!}{
    \begin{tabular}{lccccccccc}
    \toprule
         \rowcolor{myGray} & $\eta$ & $wd$ & $d_\text{hidden}$&  $d$&  $T$& $n_b$ &$\gamma$ & $\epsilon^2$ & $s$\\
         \midrule
         Flower& $1\times 10^{-4}$ & $1\times 10^{-4}$ & 512&  128&  200& 1024& 150 & 0.1 & 5\\
         DTD& $1\times 10^{-4}$ & $1\times 10^{-4}$& 512&  128&  200& 1024& 150& 0.1 & 5\\
         Pets& $1\times 10^{-4}$ & $1\times 10^{-4}$& 512&  128&  200& 1024& 150& 0.1 & 5\\
         UCF& $1\times 10^{-4}$ & $1\times 10^{-4}$& 512&  128&  200& 1024& 150& 0.1 & 5\\
         Caltech& $1\times 10^{-4}$ & $1\times 10^{-4}$& 512&  128&  200& 1024& 150& 0.1 & 5\\
         Food& $1\times 10^{-4}$ & $1\times 10^{-4}$& 512&  128&  200& 1024& 150& 0.1 & 5\\
         Aircraft& $1\times 10^{-4}$ & $1\times 10^{-4}$& 512&  128&  200& 1024& 150& 0.1 & 5\\
         EuroSAT& $1\times 10^{-4}$ & $1\times 10^{-4}$& 512&  128&  200& 1024& 150& 0.1 & 5\\
         Cars& $1\times 10^{-4}$& $1\times 10^{-4}$ & 512&  256&  1000& 1024& 200& 0.1 & 5\\
         SUN& $1\times 10^{-4}$ & $1\times 10^{-4}$& 2048&  512&  2000& 1024& 350& 0.1 & 5\\
         ImageNet-1k& $1\times 10^{-4}$ & $1\times 10^{-4}$& 8192&  1024&  4000& 8192& 600& 0.1 & 5\\
          \bottomrule
    \end{tabular}
    }
\end{center}

\end{table*}

{
Here, we present a simple and principled rule for tuning hyperparameters:
\begin{enumerate}
    \item Adjusting the output dimension $d$.
    DeepMORSE encourages representations from different classes to occupy orthogonal subspaces. Thus, the output dimension must be greater than the number of classes to accommodate these subspaces.
    \item Adjusting the balancing hyperparameter $\gamma$.
    As justified in \cite{Meng:ICLR25}, the upper bound of $\gamma$ scales linearly with $\alpha=d/(n_b\epsilon^2)$, where $d$ is the output dimension, $n_b$ is the batch-size, and $\epsilon$ is the coding precision. Since that DeepMORSE uses the same $n_b$ and $\epsilon$ for all experiments, $\gamma$ scales linearly with the output dimension $d$.
    Empirically, our settings for $\gamma$ and $d$ approximately follow $\gamma\approx0.5 \times d + 80$, which provides a simple rule for other datasets without manual search.
\end{enumerate}
}

\myparagraph{Zero-Shot Classification}
For a fair comparison with existing approaches, we directly use the image embeddings provided by ECALP~\cite{Li:ICLR25}, which are extracted using the pretrained CLIP ViT-B/16 model.
Since the dataset categories are known in zero-shot classification, we construct the text embedding dictionary $\mathcal{D}=\{\boldsymbol{\pi}_1,\cdots,\boldsymbol{\pi}_{C}\}$ by extracting text embeddings of the prompts ``A photo of {class}'' with a pretrained CLIP text encoder, where $C$ denotes the number of classes.
The textual counterpart of each image embedding is then generated using MP.
For training DeepMORSE, the hyperparameter configuration is kept exactly the same as in the image retrieval task.

{
\myparagraph{Reproducing PRO-DSC}
The original PRO-DSC~\cite{Meng:ICLR25} uses ViT-L/14 as the image encoder, while all our CLIP-based baselines in Table~\ref{tab:tab1} use ViT-B/32. To ensure a fair comparison, we reproduce its performance with ViT-B/32 using the code provided by the author and report the best clustering results after tuning hyperparameters over $\gamma\in\{0.5,1,5,10,50,100,150,300,500\}$ and $\beta\in\{0,100,200,300,400\}$.

\myparagraph{Ablation Study}
The intention of the first two rows of Table~\ref{tab:ablation} is to analyze uni-modal counterparts of DeepMORSE.
The first row is the image-only counterpart of DeepMORSE: we keep only the loss terms for the image modality, compute $\mC$ as $\mathcal{P}(\mZ_\text{img}^\top\mZ_\text{img})$ with the learned image representations, and perform spectral clustering on the affinity matrix induced by $\mA=|\mC|+|\mC^\top|$.
The second row is the text-only counterpart of DeepMORSE: we keep only the loss terms for the text modality, compute $\mC$ as $\mathcal{P}(\mZ_\text{text}^\top\mZ_\text{text})$ with the learned text representations, and perform spectral clustering on the affinity matrix induced by $\mA=|\mC|+|\mC^\top|$.
We report the best clustering results after tuning hyperparameters over $\gamma\in\{0.5,1,5,10,50,100,150,300,500\}$.
}

\end{document}